%% file: MultipleAssociationHyp_arxiv.tex
\documentclass[letter, 10pt, conference]{ieeeconf}      % Use this line for a4 paper

\IEEEoverridecommandlockouts                              % This command is only needed if
\usepackage[pdftex,dvipsnames]{xcolor}
\usepackage[normalem]{ulem}
\usepackage{amsmath}
\usepackage{amssymb}
\usepackage{bm}
\usepackage{nicefrac}
\usepackage[nohyperlinks, printonlyused, withpage, nolist]{acronym}
\usepackage{diagbox}
\usepackage{booktabs}
\usepackage{pdfpages}
\if CLASSOPTION compsoc
  \usepackage[caption=false,font=normalsize,labelfont=sf,textfont=sf]{subfig}
\else
  \usepackage[caption=false,font=footnotesize]{subfig}
\fi

\usepackage{tikz}
\usepackage{pgfplots} %tikzStuff
\pgfplotsset{compat=newest} %tikzStuff
\usetikzlibrary{plotmarks}
\usetikzlibrary{arrows.meta}
\usetikzlibrary{shapes.geometric,shapes.symbols,fit,positioning,shadows,calc}
\usepgfplotslibrary{patchplots}
\newlength\figW
\usepackage{xargs}
\usepackage[colorinlistoftodos,prependcaption,textsize=tiny]{todonotes}
\newcommandx{\unsure}[2][1=]{\todo[linecolor=red,backgroundcolor=red!25,bordercolor=red,#1]{#2}}
\newcommandx{\change}[2][1=]{\todo[linecolor=blue,backgroundcolor=blue!25,bordercolor=blue,#1]{#2}}
\newcommandx{\info}[2][1=]{\todo[linecolor=OliveGreen,backgroundcolor=OliveGreen!25,bordercolor=OliveGreen,#1]{#2}}
\newcommandx{\improvement}[2][1=]{\todo[linecolor=Plum,backgroundcolor=Plum!25,bordercolor=Plum,#1]{#2}}
\newcommandx{\thiswillnotshow}[2][1=]{\todo[disable,#1]{#2}}

\title{\LARGE \bf
LMB Filter Based Tracking Allowing for Multiple Hypotheses in Object Reference Point Association*
}

\author{Martin Herrmann$^{1}$, Aldi Piroli$^{1}$, Jan Strohbeck$^{1}$, Johannes M\"uller$^{1}$ and Michael Buchholz$^{1}$% <-this % stops a space
\thanks{*Part of this work was financially supported by the Federal Ministry of Economic Affairs and Energy of Germany within the program "Highly and Fully Automated Driving in Demanding Driving Situations" (project MEC-View, grant number 19A16010I).}% <-this % stops a space
\thanks{*Part of this work has been conducted as part of ICT4CART project which has received funding from the European Union's Horizon 2020 research \& innovation programme under grant agreement No. 768953. Content reflects only the authors' view and European Commission is not responsible for any use that may be made of the information it contains.
}% <-this % stops a space
\thanks{$^{1}$The authors are with Institute of Measurement, Control, and Microtechnology, Ulm University, Germany
{\tt\footnotesize {\{firstname.lastname\}}@uni-ulm.de} and {\tt\footnotesize {johannes-christian.mueller}@uni-ulm.de}}%
}

\newtheorem{proposition}{Proposition}[section]

\newcommand\copyrighttext{%
    \footnotesize Copyright $\copyright$ 2020 IEEE.
    Personal use of this material is permitted.
    Permission from IEEE must be obtained for all other uses, in any current or future media, including reprinting/republishing this material for advertising or promotional purposes, creating new collective works, for resale or redistribution to servers or lists, or reuse of any copyrighted component of this work in other works.}%
\newcommand\copyrightnotice{%
    \begin{tikzpicture}[remember picture,overlay]%
    \node[anchor=south,yshift=10pt] at (current page.south) {\fbox{\parbox{\dimexpr\textwidth-2cm}{\copyrighttext}}};%
    \end{tikzpicture}%
    \vspace{-10pt}%
}

\begin{document}

\maketitle
\copyrightnotice
\thispagestyle{empty}
\pagestyle{empty}

%%%%%%%%%%%%%%%%%%%%%%%%%%%%%%%%%%%%%%%%%%%%%%%%%%%%%%%%%%%%%%%%%%%%%%%%%%%%%%%%

\begin{acronym}
\acro{LMB}[LMB]{Labeled Multi-Bernoulli}
\acro{V2V}[V2V]{Vehicle-to-Vehicle}
\acro{V2X}[V2X]{Vehicle-to-Anything}
\acro{MEC}[MEC]{Multi-access Edge Computing}
\acro{ETSI}[ETSI]{European Telecommunications Standards Institute}
\acro{UE}[UE]{User Equipment}
\acro{eNB}[eNB]{evolved Node B}
\acro{5G}[5G]{5th Generation}
\acro{CAM}[CAM]{Cooperative Awareness Message}
\acro{CPM}[CPM]{Collective Perception Message}
\acro{FOV}[FOV]{Field of View}
\acro{FISST}[FISST]{Finite Set Statistics}
\acro{RFS}[RFS]{Random Finite Set}
\acro{JIPDA}[JIPDA]{Joint Integrated Probabilistic Data Association}
\acro{MHT}[MHT]{Multi-Hypothesis Tracking}
\acro{PHD}[PHD]{Probability Hypothesis Density}
\acro{CPHD}[CPHD]{Cardinalized Probability Hypothesis Density}
\acro{CBMEMBER}[CB-MeMBer]{Cardinality Balanced Multi-Target Multi-Bernoulli}
\acro{GLMB}[GLMB]{Generalized Labeled Multi-Bernoulli}
\acro{UKF}[UKF]{Unscented Kalman Filter}
\acro{EKF}[EKF]{Extended Kalman Filter}
\acro{CTRV}[CTRV]{Constant Turn Rate and Velocity}
\acro{CTRA}[CTRA]{Constant Turn Rate and Acceleration}
\acro{GPS}[GPS]{Global Positioning System}
%\acro{MEC-View}[MEC-View]{MEC-View}%Mobile Edge Computing basierte Objekterkennung f\"ur hoch- und vollautomatisiertes Fahren}
%\acro{ICT4CART}[ICT4CART]{ICT4CART}%Information and Communication Technology Infrastructure for Connected and Automated Road Transport}
\acro{LTE}[LTE]{Long Term Evolution}
\acro{LTE-A}[LTE-A]{LTE-Advanced}
\acro{CI}[CI]{Covariance Intersection}
\acro{GCI}[GCI]{Generalized Covariance Intersection}
\acro{ITS}[ITS]{Intelligent Transportation System}
\acro{G5}[ITS-G5]{ITS-G5}
\acro{OSPAT}[OSPAT]{Optimal Sub-pattern Assignment metric for Track}
\acro{MSE}[MSE]{Mean Squared Error}
\acro{iid}[i.i.d.]{independent and identically distributed}
\acro{MTT}[MTT]{Multi Target Tracking}
\acro{MAX}[MAX]{Single Hypothesis Method with Maximum Likelihood Decision}
\acro{MH}[MH]{Multiple Hypothesis}
\acro{MEAS}[MEAS]{Single Hypothesis Method with exactly measured object reference point}
\acro{MHC}[MHC]{Multiple Hypothesis with Constraints Check}
\acro{AV}[AV]{Autonomous Vehicle}
\end{acronym}

%%%%%%%%%%%%%%%%%%%%%%%%%%%%%%%%%%%%%%%%%%%%%%%%%%%%%%%%%%%%%%%%%%%%%%%%%%%%%%%%

\begin{abstract}
Autonomous vehicles need precise knowledge on dynamic objects in their surroundings. Especially in urban areas with many objects and possible occlusions, an infrastructure system based on a multi-sensor setup can provide the required environment model for the vehicles. Previously, we have published a concept of object reference points (e.g. the corners of an object), which allows for generic sensor "plug and play" interfaces and relatively cheap sensors. This paper describes a novel method to additionally incorporate multiple hypotheses for fusing the measurements of the object reference points using an extension to the previously presented Labeled Multi-Bernoulli (LMB) filter. In contrast to the previous work, this approach improves the tracking quality in the cases where the correct association of the measurement and the object reference point is unknown. Furthermore, this paper identifies options based on physical models to sort out inconsistent and unfeasible associations at an early stage in order to keep the method computationally tractable for real-time applications. The method is evaluated on simulations as well as on real scenarios. In comparison to comparable methods, the proposed approach shows a considerable performance increase, especially the number of non-continuous tracks is decreased significantly.
\end{abstract}

%%%%%%%%%%%%%%%%%%%%%%%%%%%%%%%%%%%%%%%%%%%%%%%%%%%%%%%%%%%%%%%%%%%%%%%%%%%%%%%%

\section{INTRODUCTION} \label{introduction}
Accurate modeling of dynamic objects in the environment of \ac{AV} is an essential task for their reliable and safe operation. However, despite of many research carried out in that field, there are still many open questions, e.g. regarding the computational tractability, the level of detail objects are to be described, or the method of choice in tracking applications \cite{ADSSurvey2019}, especially in mass market scenarios where cheap sensors are used \cite{SurveyAutonomous2019}. The early stage of detection and tracking of dynamic objects, e.g. other vehicles, pedestrians or other road users, though, is of great importance, since errors propagate through the system and can lead to failures of later stages.

Besides an environment detection by sensors on-board an \ac{AV}, as described e.g. in \cite{sandy}, infrastructure sensors can be an additional source of information to assist the \ac{AV} especially in complex scenarios like urban intersections \cite{digital_mirror}. As realized in the project MEC-View~\cite{mecview}, this even allows for merging scenarios without a direct line of sight \cite{MuellerPlanner2019}.

Realizing such a functionality poses multiple requirements to the infrastructure system, like low latency or reliability \cite{MuellerIVSoTIF2019}, and the perception task itself. Besides the number and position of objects, also information on their type, their dynamic state and their extent state are crucial, e.g. to determine if a gap between two vehicles is suitable to merge into it. Fig.~\ref{fig:van_trailer} shows an example of the infrastructure based system of \cite{mecview}, which demonstrates the challenge of the estimation of the extent in occluded scenarios. Even in case of perfect detectors, it is impossible for the first camera to predict the existence of a trailer, which can be seen in the image of the second camera. Thus, camera 1 can not measure the length of the van or, in the worst case, delivers biased measurements. However, unbiased measurements are a basic requirement for \ac{MTT} approaches, whose defiance introduces additional errors. Using the \ac{MTT} filter proposed by us in \cite{env_model_rp}, such scenarios can be handled naturally by only using pure position measurements of the object reference points (corners). Performance issues in the special case of unknown association of such measurements and the object reference points are addressed in this paper and solved by the proposed extension.

\begin{figure}%
  \vspace{-3pt}%
  %\centering%
  \subfloat[Frontal view of camera 1]{%
    \includegraphics[width=.5\columnwidth]{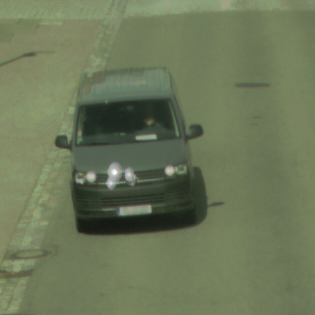}%
  }
  \subfloat[Side view of camera 2]{%
    \includegraphics[width=.5\columnwidth]{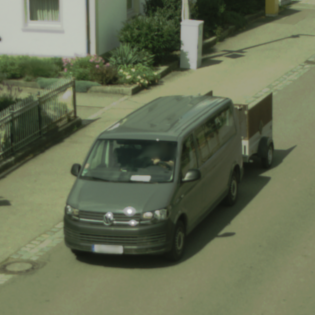}%
  }
  \caption{Challenging scenario due to occlusion of the trailer in camera 1.}
  \label{fig:van_trailer}
  \vspace*{-15pt}%
\end{figure}

\subsection{Related Work} \label{related_work}
There are multiple approximative methods available for tracking of multiple dynamic objects with multiple sensors, as \ac{MHT} \cite{mht_tracking}, \ac{PHD} \cite{Mahler2003}, or \ac{CBMEMBER} \cite{CBMEMBER09}. A rigorous extension of the single object Bayes filter \cite{Ho64} is known as the Multi-Object Bayes filter \cite{Mahler2007}, which, similar to the Kalman filter, has an analytical solution under specific circumstances, namely the \ac{GLMB} filter proposed by Vo and Vo \cite{Vo2013}. The \ac{LMB} filter \cite{Reuter2014a} is an approximation of the latter, which constitutes a good compromise between computational tractability and performance.

In the case of extended targets, where the extent of objects cannot be neglected, the extent itself or the parameters of a model describing the extent or shape of an object have to be estimated, too. \cite{extendedTargetTracking2017} gives an overview on extended object tracking and basically distinguishes between two different strategies. First, there are methods that estimate the shape of an object based on single measurements, which is tracked and refined over time. Often simple models, as rectangles or circles are used, but also more complex ones. Hereby, the object state is augmented by the model parameters which need to be estimated by the sensors. Second, methods known as extended target tracking are available, which differ from the latter by allowing an object to give rise to more than one measurement. The estimation of an object's shape from a single measurement, though, is often not possible. However, both strategies require high computational effort.

Our approach \cite{env_model_rp} combines ideas of both strategies. While the method sticks to the standard multi-object measurement model that prohibits multiple measurements per object, it does not require the estimation of the shape of an object in a single measurement. Nevertheless, the estimation of an object's extent over time is possible without any sensor being capable to do so on its own. This is based on the highly flexible concept that the extent of a rectangle can be estimated from the knowledge of the corner positions of an object. Thus, the problem of biased extent measurements is obsolete.

A method to consider multiple measurement hypotheses is proposed in \cite{MultipleHypothesis2018}, which is mathematically similar to our approach. While their approach treats the problem that typical preprocessing algorithms sometimes violate the requirements of the standard measurement model, our approach tackles the problem of unknown association of measurements to object reference points. Both approaches may be used together and complement each other.

\subsection{Contributions}
This paper bases on the previously published work \cite{env_model_rp} and has a twofold contribution. First, the performance weaknesses of the original method in situations of large noise and unknown measurement to object reference point association are discussed and second, a substantial improvement is presented, which overcomes those. It is shown that the use of the proposed method reduces the number of non-continuous trajectories tremendously, which is due to a probabilistic modeling of multiple reference point associations within the filter update. This prevents the filter from incorporation of erroneous information due to incorrect association, and allows for continuous trajectories even in case of ambiguous associations. The proposed extensions are mathematically proven and evaluated in a series of Monte Carlo simulations and a complex real scenario.

The present paper recapitulizes the fundamental idea of the previous \ac{LMB} filter in Section \ref{MOT_tracking}, which also elaborates the challenges of scenarios with ambiguous generic measurements. In Section \ref{mult_hyp} we present our novel extension to the mentioned filter that is able to cope with such situations. Finally, our paper discusses ideas to reduce the increased computational costs of the presented method in Section \ref{sec:validation_gate} and closes with an extensive evaluation on simulated and real data in Section \ref{evaluation}, as well as a conclusive review in Section \ref{conclusions}.

%%%%%%%%%%%%%%%%%%%%%%%%%%%%%%%%%%%%%%%%%%%%%%%%%%%%%%%%%%%%%%%%%%%%%%%%%%%%%%%%

\section{Multi-Object Tracking with Object Reference Point Association} \label{MOT_tracking}
The proposed method builds upon the specific measurement model of \cite{env_model_rp}, which is referred to as \acf{MAX} in the following and summarized here shortly, especially emphasizing points which are important for the presented extensions.\footnote{Note that albeit the presented approach can easily be adapted to other multi-sensor multi-object filters, derivations base on the \ac{LMB} filter, hence notation follows \cite{Vo2013}. Further, note that the structure of $\underline{H}^{(\zeta)}(x)$ and definition of $\mathbb{A'}$, depend on the arrangement of the vectors $x$ and $z$, as well as on the dimension of the overall coordinate system. Adaptions to applications with distinct settings, e.g. a three dimensional coordinate system, are easily possible.}

\subsection{Object Model} \label{object_model}
All objects are described by a probabilistic box model, using multivariate Gaussian mixtures \cite{McLachlan2000} within a two-dimensional coordinate frame, based on the assumption of a flat world. Such object is completely described by the feature vector $o = \left[p(\hat{x}), \zeta, r, \ell, k \right]^T$, where the probability density function $p(\hat{x})$ describes a box's spatial state and $\zeta \in \mathbb{A}'$ is the object reference point, which describes the position at which the spatial distribution is originated. $\mathbb{A}'$ consists of $\mathbb{A}' = \{\mathsf{C}, \mathsf{FL}, \mathsf{FR}, \mathsf{BL}, \mathsf{BR}\}$, thus the geometrical center, front left, front right, back left and back right corner. Further, $r = p(\exists x)$ denotes the existence probability of an object, $\ell$ is a unique label and $k$ is a time index. The spatial distribution follows
\begin{equation}
p(x) = \sum_{j=1}^{J}w^{(j)}\mathcal{N}\left(x; \hat{x}^{(j)}, \underline{P}^{(j)}\right)\, ,
\end{equation}

where the weights sum up to one $\sum_{j=1}^{J}w^{(j)} = 1$ and $0 \leq w^{(j)} \leq 1$. Hereby, all mixture components are independent, the elements of the state vectors $x \in \mathbb{R}^8$ are $\hat{x} = [x^{(\zeta)}, y^{(\zeta)}, \varphi, \varphi', v, a, w, l]^T$, and $\underline{P}\in \mathbb{R}^{8\times8}$ is the respective covariance matrix. $x^{(\zeta)}$ and $y^{(\zeta)}$ describe the position of the object reference point, $\varphi$ is the yaw angle and $\varphi'$ the yaw rate, $v$ the absolute velocity, $a$ the absolute acceleration and $w$ and $l$ the width and length, respectively. Note that within the \ac{LMB} filter all objects are transformed into the center object reference point in order to ensure correct prediction.

\subsection{Measurements} \label{measurements}
Sensor measurements are modeled very similar to objects using a probabilistic box model, though using a multivariate Gaussian probability density function and allowing for \textit{incomplete} measurements. This means all features in a sensor measurement are optional except the position $[x^{(\zeta)}, y^{(\zeta)}]^T$ (where $\zeta$ may be unknown) and the time index $k$, of course. Thus, a measurement is described by the feature vector $o = \left[p(\hat{x}), \zeta, r, \ell, k \right]^T$, too, but the spatial distribution is defined as $p(\hat{x}) = \mathcal{N}\left(x; \hat{x}, \underline{P}\right)$ and appropriate features being optional. Further, note that all measured features must be complemented by an estimation of their uncertainty.

\subsection{Structural Aspects} \label{structural_aspects}
In order to enable estimation of the extent of dynamic objects only two conditions have been made, of which each is sufficient. Either a sensor is able to reliably measure all questioned features or the object in question is observed in all object reference points by a set of sensors. Moreover, the rules of Bayesian inference and \ac{FISST} apply, especially the conditions of the measurement model have to be met, e.g. an object gives rise to at most one measurement. Overall, a unilateral communication takes place, where sensors send preprocessed measurements to the centralized \ac{LMB} filter, as shown in Fig. 2 in \cite{env_model_rp}.

\subsection{Multi-Object Multi-Sensor Tracking using a Labeled Multi-Bernoulli filter} \label{lmb_pre}
The standard multi-object likelihood \cite{Vo2014} used within the \ac{LMB} filter \cite{Reuter2014a} is defined as
\begin{equation}
g(Z|\textbf{X}) = e^{-\langle \kappa, 1\rangle}\kappa^Z\sum_{\theta \in \Theta(\mathcal{L}(\textbf{X}))}\left[\psi_Z(\cdot;\theta)\right]^\textbf{X}\, ,
\end{equation}
where
\begin{equation}
\psi_Z(x,\ell;\theta) = \begin{cases}
	\frac{p_D(x,\ell)g(z_{\theta(\ell)}|x,\ell)}{\kappa(z_{\theta(\ell)})}, & \text{if } \theta(\ell) > 0 \\
	1-p_D(x,\ell), & \text{if } \theta(\ell) = 0\, .
\end{cases}
\end{equation}

In \cite{env_model_rp} the single object likelihood $g(z|x)$ of a measurement $z$ given an object with state $x$ was defined by
\begin{equation} \label{single_meas_likelihood_prior}
g(z|x) = \underset{\zeta \in \mathbb{A}}{\mathrm{max}}\,\,\mathcal{N}\left(z; \underline{H}^{(\zeta)} x, \underline{R}\right) \, ,
\end{equation}
where $\mathbb{A} = \mathbb{A}' \setminus \mathsf{C}$ or $\mathbb{A} = {\zeta}$, if $\zeta$ is specified by the sensor.

This measurement model evaluates the probability of all possible object reference point associations and decides for that $\zeta$ which maximizes the likelihood of the measurement and predicted measurement association, based on the Mahalanobis distance \cite{mahalanobis} between both.

Actually, the measurement matrix $\underline{H}^{(\zeta)}(x) \in \mathbb{R}^{d \times n}$ is a function, since it depends on the yaw angle of the object, and is non-linear due to this dependency. It specifies the transformation of a probability density from the center to a corner object reference point and is used to calculate the predicted measurement $z_+^{(\ell,j,\zeta)}$. Hereby, $d$ is the number of measured features and $n$ the overall number of features of the mean vector.

In order to construct the measurement matrix $\underline{H}^{(\zeta)}$, the full transformation matrix $\underline{\tilde{H}}^{(\zeta)}(x) \in \mathbb{R}^{n \times n}$ is defined first, by
\begin{equation}\label{HTrafo}
  \underline{\tilde{H}}^{(\zeta)}(x) =
  \begin{bmatrix}
    \underline{I} & \underline{\Delta}^{(\zeta)}(x) \\
    \underline{0} & \underline{I}
  \end{bmatrix} \, ,
\end{equation}
where $\underline{\Delta}^{(\zeta)}(x) \in \mathbb{R}^{2 \times n-2}$ is responsible for the correlation between position and extent of an object and is
\begin{equation}
  \underline{\Delta}^{(\zeta)}(x) =
  \begin{bmatrix}
    \underline{0} & \underline{f}^{(\zeta)}(x)
  \end{bmatrix} \, ,
\end{equation}
with
\begin{equation}
  \underline{f}^{(\zeta)}(x) = \frac{1}{2} \cdot
  \begin{bmatrix}
  	-\sin{(\varphi)} & \cos{(\varphi)} \\
    \cos{(\varphi)} & \sin{(\varphi)}\\
  \end{bmatrix}  \circ
  \begin{bmatrix}
  	\delta & \gamma \\
    \delta & \gamma\\
  \end{bmatrix} \, ,
\end{equation}
where $\circ$ denotes the element-wise multiplication and
\begin{equation}
  \delta =
  \begin{cases}
    1 & \text{if } \zeta = \mathsf{BL}, \mathsf{FL} \, ,\\
    -1 & \text{if } \zeta = \mathsf{BR}, \mathsf{FR} \, ,\\
    % 0 & \text{if } \zeta = \mathsf{C} \, ,\\
  \end{cases} \,
  \gamma =
  \begin{cases}
    1 & \text{if } \zeta = \mathsf{FL}, \mathsf{FR} \, ,\\
    -1 & \text{if } \zeta = \mathsf{BL}, \mathsf{BR}\, .\\
    % 0 & \text{if } \zeta = \mathsf{C} \, .\\
  \end{cases}
\end{equation}

Finally, if $d < n$, i.e. a sensor measures a subset of the full state space only, the respective rows of $\underline{\tilde{H}}^{(\zeta)}$ are to be deleted to obtain $\underline{H}^{(\zeta)}$.

\subsection{Challenges due to Generic Measurements} \label{rp_err_ex}
Evaluations in \cite{env_model_rp} showed that the \ac{MAX} method performs poor in noisy situations where sensors provide position measurements only, i.e. the feature vector equals $o = \left[\hat{x}, \underline{P}, k \right]^T$ with $\hat{x} = \left[x^{(\zeta)}, y^{(\zeta)}\right]^T$, where the object reference point $\zeta$ is not specified. Since the determination of the object reference point in the innovation process of the \ac{LMB} filter is irreversible, erroneous decisions propagate through time and degrade the performance of the \ac{LMB} filter tremendously. This effect can be seen in Fig. \ref{fig:assocexample}, which shows a vehicle (green), moving to the right and a sensor under the influence of heavy noise, that measures the position of the vehicle's front left corner (red). Both, the upper and lower image, show the same situation, but the association of the first measurement is done differently. In the upper image the maximization leads to an \textit{erroneous} association of the measurement to the front right corner, while the \textit{true} association to the front left corner is shown in the lower image. Obviously, the \textit{erroneous} association leads to a large error in the estimation (blue) in the upper image, while the estimation in the lower image is only shifted and twisted slightly. In the end, the overall Mahalanobis distance in the lower image is smaller than in the upper due to the large extent error. Even worse, the probability of subsequent measurements falling out of the gating range is increased in the upper image, which can lead to track termination and non-continuous trajectories.

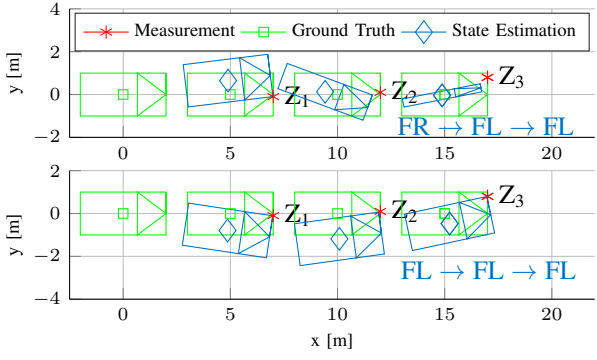
\begin{figure}
  \vspace{2pt}
  %\centering
  \subfloat{
    \setlength\figW{0.85\columnwidth}
    \input{assoc_example_false.tex}
    \label{fig:assocexample_false}}
  \hfill
  \vspace*{-25pt}
  %\centering
  \subfloat{
    \setlength\figW{0.85\columnwidth}
    \input{assoc_example_corr.tex}
    \label{fig:assocexample_corr}}
  \vspace{-5pt}
  \caption{The scenario points out the effect of an erroneous track to measurement association, as it can happen in challenging scenarios and demonstrates the presented approach.}%
  \label{fig:assocexample}
  \vspace*{-15pt}
\end{figure}

Very similar effects occur if the detector's specification of the object reference point $\zeta$ is inaccurate.

%%%%%%%%%%%%%%%%%%%%%%%%%%%%%%%%%%%%%%%%%%%%%%%%%%%%%%%%%%%%%%%%%%%%%%%%%%%%%%%%

\section{Incorporation of Multiple Object Reference Point Association Hypotheses} \label{mult_hyp}
In order to tackle the performance problems of the filter in case of unknown association of measurements to an object reference point, our proposed method, referred to as \ac{MH} method, avoids the deterministic choice on a specific object reference point. It rather incorporates all possible object reference points and is defined as
\begin{equation} \label{single_meas_likelihood}
g(z|x) = \sum_{\zeta \in \mathbb{A}}w^{(\zeta)}\mathcal{N}\left(z; \underline{H}^{(\zeta)} x, \underline{R}\right) \, .
\end{equation}

Here, $\mathbb{A} = \mathbb{A}' \setminus \mathsf{C}$ denotes the set of object reference points whose association is in question and $\underline{H}^{(\zeta)}$ is the respective measurement matrix, which transforms the spatial distribution $p(x,\ell)$ of a track $\ell$ into the respective object reference point $\zeta$ (see \eqref{HTrafo}). Obviously, if $\zeta$ is specified by a sensor, \ac{MH} and \ac{MAX} are equal.

\begin{proposition}
If the likelihood for a measurement $z$ given a single object with state $x$ follows the form of \eqref{single_meas_likelihood}, the updated posterior density yields
\begin{align} \label{new_upd}
p^{(\theta)}&(x,\ell|Z)\nonumber\\
&= \sum_{j=1}^{J_+^{(\ell)}} \sum_{\zeta \in \mathbb{A}} w^{(\ell,j,\theta,\zeta)}(Z) \mathcal{N}\left(x; \hat{x}^{(\ell,j,\theta,\zeta)}, \underline{P}^{(\ell,j,\zeta)}\right)\, ,
\end{align}
where
\begin{align}
w^{(\ell,j,\theta,\zeta)}&(Z)\nonumber \\
= &\frac{p_D(x,\ell)}{\kappa(z_{\theta(\ell)})\eta_Z^{(\theta)}(\ell)} w_+^{(\ell,j)}w^{(\zeta)} \mathcal{N}\left(z_{\theta(\ell)}; z_+^{(\ell,j,\zeta)}, \underline{S}^{(\ell,j,\zeta)}\right)
\end{align}
and
\begin{align}
\underline{S}^{(\ell,j,\zeta)} &= \underline{H}^{(\zeta)} \underline{P}_+^{(\ell,j)} \left[\underline{H}^{(\zeta)}\right]^T + \underline{R}\, ,\\
z_+^{(\ell,j,\zeta)} &= \underline{H}^{(\zeta)} \hat{x}_+^{(\ell,j)}\, , \\
\hat{x}^{(\ell,j,\theta,\zeta)} &= \hat{x}_+^{(\ell,j)} + \underline{K}^{(\ell,j,\zeta)} \left( z_{\theta(\ell)} - z_+^{(\ell,j,\zeta)} \right)\, , \\
\underline{K}^{(\ell,j,\zeta)} &= \underline{P}_+^{(\ell,j,\zeta)} \left[\underline{H}^{(\zeta)}\right]^T \left[ \underline{S}^{(\ell,j,\zeta)} \right]^{-1}\, ,\\
\underline{P}^{(\ell,j,\zeta)} &= \underline{P}_+^{(\ell,j,\zeta)} - \underline{K}^{(\ell,j,\zeta)} \underline{S}^{(\ell,j,\zeta)} \left[ \underline{K}^{(\ell,j,\zeta)} \right]^{T}\, .
\end{align}
\end{proposition}
\begin{proof}
If the predicted spatial density of track $x$ with label $\ell$ follows a Gaussian mixture, the updated posterior spatial distribution, following \cite{Reuter2014a} equals
\begin{equation} \label{posterior_lmb}
p^{(\theta)}(x,\ell|Z) = \frac{p_+(x,\ell) \cdot \psi_Z(x,\ell;\theta)}{\eta_Z^{(\theta)}(\ell)} \, ,
\end{equation}
with the generalized measurement likelihood for object x and association $\theta$ being
\begin{equation} \label{generalized_likelihood}
\psi_Z(x, \ell; \theta) = \frac{p_D(x,\ell)g(z_{\theta(\ell)}|x,\ell)}{\kappa(z_{\theta(\ell)})} \, ,
\end{equation}
and the normalization constant
\begin{equation}
\eta_Z^{(\theta)}(\ell) = \langle p_+(\cdot,\ell), \psi_Z(\cdot,\ell;\theta) \rangle \, .
\end{equation}
Substituting \eqref{single_meas_likelihood} into \eqref{posterior_lmb} and \eqref{generalized_likelihood} and using the Gaussian identities of \cite{Ho64} yields
\begin{align} \label{proofed_new_upd}
p^{(\theta)}&(x,\ell|Z) \nonumber \\
&= \frac{p_D(x,\ell)}{\kappa(z_{\theta(\ell)})\eta_Z^{(\theta)}(\ell)} \sum_{j=1}^{J_+^{(\ell)}}w_+^{(\ell,j)}\mathcal{N}\left(x; \hat{x}_+^{(\ell,j)}, \underline{P}_+^{(\ell,j)}\right) \nonumber \\
&\sum_{\zeta \in \mathbb{A}}w^{(\zeta)}\mathcal{N}\left(z; \underline{H}_\zeta x, \underline{R}\right) \\
&= \sum_{j=1}^{J_+^{(\ell)}} \sum_{\zeta \in \mathbb{A}} w^{(\ell,j,\theta,\zeta)}(Z) \mathcal{N}\left(x; \hat{x}^{(\ell,j,\theta,\zeta)}, \underline{P}^{(\ell,j,\zeta)}\right)\, .
\end{align}

In contrast to the standard form of the \ac{LMB} updated posterior density the number of mixture components of the spatial distribution is increased within the innovation step. But \eqref{new_upd} still represents a Gaussian mixture since it can be rewritten
\begin{align}
p^{(\theta)}&(x,\ell|Z) = \sum_{j=1}^{\tilde{J}^{(\ell)}} w^{(\ell,j,\theta)}(Z) \mathcal{N}\left(x; \hat{x}^{(\ell,j,\theta)}, \underline{P}^{(\ell,j)}\right)\, ,
\end{align}
where $\tilde{J}^{(\ell)} = J_+^{(\ell)} \cdot |\mathbb{A}|$. The updated posterior density is thus, still of the form of an \ac{LMB} distribution.
\end{proof}

If no further information is available the weight of all associations is equally set to $w^{(\zeta)} = (|\mathbb{A}|)^{-1}$.

\section{Truncation of Tracks via Validation Gate}\label{sec:validation_gate}
If not truncated, the number of mixture components increases rapidly even in the standard \ac{LMB} filter. The proposed measurement model, however, further increases the number of mixture components and depends on the survival of many components, since the idea is to keep even improbable components for some time to overcome weak sensor measurements. It is, thus, all the more important to develop a strategy to reduce the number of mixture components to a tolerable level. To do so, we suggest using a twofold strategy based on a traditional pruning of components and truncation via a model based gating, which we call validation gate.

\subsection{Pruning of Mixture Components}
Investigation of the updated posterior densities of the proposed filter have shown that most mixture components are very close to each other and can be merged or pruned. However, in cases of large noise, where false associations lead to high component weights of \textit{false} components, the \textit{correct} components are often separated clearly from others. Hence, it showed that setting the pruning level to very low values, while setting a moderately high maximal merging distance constitutes a good compromise, where a sufficient suppression of components takes place, while unlikely, but important components survive.

\subsection{Validation Gate}
Irrespective of other truncation and pruning strategies we propose a special kind of gating, named validation gate. The idea bases on the ideas of measurement validation procedures \cite{BarShalom2011}, where, based on the Mahalanobis distance between measurements and possibly associated tracks, a gating region is defined, within which a measurement has to lie to be considered. The validation gate evaluates the posterior density after the innovation and compares the density with model constraints. Association of measurements to object reference points, which lead to invalid posterior densities are thrown away. Therefore, the measurement model \eqref{single_meas_likelihood} is adapted, such that
\begin{equation}
\tilde{\mathbb{A}} = \{ \zeta \in \mathbb{A} : f_u(x) \neq 0 \,\, \forall u \in \mathbb{U} \} \,
\end{equation}
where $\mathbb{U}$ is the set of all model constraints that apply.

Hereby, $f(x)$ is a function, which tests the respective estimate $\hat{x}^{(\ell,j,\theta,\zeta)}$ on arbitrary model constraints. The following is a set of such, which apply to vehicles and are used within the evaluation of this paper.

\subsubsection{Vehicle Extent}
The extent of a vehicle is, of course, lower bounded. Due to the coupling of the position and the object's extent, however, the updated posterior density may have a negative length or width (in 3D extensions the height may also be negative). Thus, the evaluation function is
\begin{align}
f_{e}(x) = \begin{cases}
   1 &\text{if } x(w) \geq 0 \text{ and } x(l) \geq 0 \\
   0 &\text{otherwise} \,.
\end{cases}
\end{align}

\subsubsection{Vehicle Extent Ratio}
The length of typical consumer road vehicles is much larger than their width. Therefore, the length to width ratio $r = \nicefrac{l}{w}$ of vehicles on the market is typically lower and upper bounded and can be used to validate a track. Thus, the evaluation function, based on vehicle dimensions of \textit{Smart Fortwo} and \textit{Ford Super Duty}, is
\begin{equation}
f_{r}(x) = \begin{cases}
   1 &\text{if } 1.5 \leq r \leq 4 \\
   0 &\text{otherwise} \,.
\end{cases}
\end{equation}

\subsubsection{Yaw Rate}
Typical vehicles are not able to turn when stationary and the yaw rate is upper bounded by the vehicle's physical dimensions and speed. Based on the single track model \cite{Paden16}, the velocity $v$ and a maximum steering angle of $\gamma = 75^{\circ}$, the minimum turning radius is calculated, from which the maximum yaw rate of $\varphi'_{\text{max}} = \nicefrac{v\gamma}{l}$ follows. Thus, the evaluation function is
\begin{equation}
f_{\varphi'}(x) =
\begin{cases}
   1 &\text{if } |x(\varphi')| \leq \varphi'_{\text{max}} \\
   0 &\text{otherwise} \,.
\end{cases}
\end{equation}
% See also https://de.wikipedia.org/wiki/Wendekreis_(Fahrzeug)

\subsubsection{Acceleration}
The acceleration of a vehicle is more or less upper bounded, too. Here, an upper limit of $a_{\text{max}} = 10\,\text{m\,s}^{-2}$ is used, which equals the average acceleration of a vehicle which accelerates from $0\,\text{km\,h}^{-1}$ to $100\,\text{km\,h}^{-1}$ in under $2.8\,s$. Thus, the evaluation function is
\begin{equation}
f_{a}(x) =
\begin{cases}
   1 &\text{if } x(a) \leq a_{\text{max}} \\
   0 &\text{otherwise} \,.
\end{cases}
\end{equation}

\subsubsection{Velocity}
The velocity of vehicles can be lower bounded from practical reasons, a lower bound of $v_{\text{max}} = -5\,\text{m\,s}^{-1}$ is assumed, which results in an evaluation function
\begin{equation}
f_{v}(x) =
\begin{cases}
   1 &\text{if } x(v) \geq v_{\text{max}} \\
   0 &\text{otherwise} \,.
\end{cases}
\end{equation}

%%%%%%%%%%%%%%%%%%%%%%%%%%%%%%%%%%%%%%%%%%%%%%%%%%%%%%%%%%%%%%%%%%%%%%%%%%%%%%%%

\section{EVALUATION} \label{evaluation}
The performance of the proposed method, considering multiple concurrent object reference point hypotheses, is evaluated in a simulation and a real world demonstration. Furthermore, an examination of the computational costs is given to assess the real time capability.

\subsection{Simulation}
In order to work out the improvements of the proposed method a simulation has been carried out, that equals the one of \cite{env_model_rp}, i.e. a five seconds lasting scenario with three vehicles, as can be seen in Fig. \ref{fig:scenario_1}, observed by three distributed sensors. The vehicles' starting points are marked by circles, the trajectories' ends by a triangle. 
As in \cite{env_model_rp}, the sensors are set up such that the measurement vector equals $o = \left[\hat{x}, \underline{P}, k \right]$ with $\hat{x} = \left[x^{(\zeta)}, y^{(\zeta)}\right]^T$, where $\zeta$ is unknown. All objects are measured in random object reference points, such that the conditions made in Section \ref{structural_aspects} hold. The particular covariance matrices of the sensor measurements are $R = \text{diag}(\sigma^2_x, \sigma^2_y)$, where $\sigma = \sigma_x = 2\sigma_y$ expresses the standard deviation of the measurement error in longitudinal and lateral direction in sensor coordinates. Furthermore, a measurement is missed by a chance of $1-p_D = 0.05$ and clutter events are Poisson distributed with the expected number of such being $\lambda_c = 0.1$, which are distributed equally over the measurement space ($[x_{\text{min}}, x_{\text{max}}, y_{\text{min}}, y_{\text{max}}] = [-30, 30, -10, 10]$). Tracks are predicted using a \ac{CTRA} model with parameters $\sigma_j = 1\,\text{m\,s}^{-3}$ (jerk), $\sigma_\varphi = 0.5\,\text{rad\,s}^{-2}$ (turn acceleration), and pseudo noise in the extent of $\sigma_e = 0.1\,\text{m}$. The scenario is repeated four times, where the measurement error is varied in steps of $0.5\,\text{m}$ from $\sigma = 0.5\,\text{m}$ to $2.0\,\text{m}$ and each scenario is repeated in a Monte Carlo simulation with 100 trials.

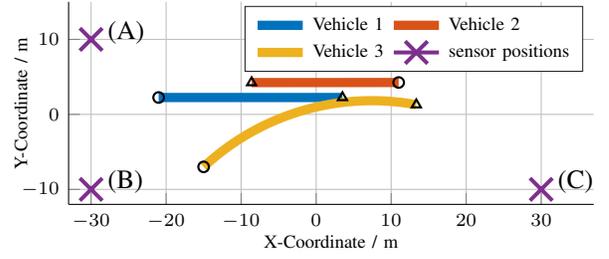
\begin{figure}
  \vspace{4pt}
  %\centering
  \subfloat{
    \setlength\figW{0.85\columnwidth}
    \input{scenario_1.tex}}
    % Created with "rectangular_3cams" launch files and "makeSscenarioPlot" and params multihyp;
  \caption{Simulation scenario with three vehicles observed by three sensors.}
  \label{fig:scenario_1}
  \hfill
  \vspace*{-15pt}
\end{figure}%

\begin{figure}
  \vspace{2pt}
  \captionsetup[subfigure]{justification=centering}
  \subfloat[OSPAT errors of \ac{MH} method versus \ac{MAX} method.]{
    \setlength\figW{0.85\columnwidth}
    \input{multihyp_vs_maxhyp_ospat.tex}
    \label{fig:multihyp_vs_maxhyp_ospat}}
  \hfill
  \vspace{-10pt}
  \subfloat[OSPAT errors of \ac{MH} method versus \ac{MEAS} method.]{
    \setlength\figW{0.85\columnwidth}
    \input{multihyp_vs_measured_ospat.tex}
    \label{fig:multihyp_vs_measured_ospat2}}
  \hfill
  \vspace{-10pt}
  \subfloat[OSPAT errors of \ac{MH} method versus \ac{MHC} method.]{
    \setlength\figW{0.85\columnwidth}
    \input{multihyp_vs_multihyp_ccheck_ospat.tex}
    \label{fig:multihyp_vs_multihyp_ccheck_ospat2}}
  \hfill
  \vspace{-1pt}
  \caption{Comparison of the \ac{OSPAT} errors for the different methods.}
  \label{fig:scenario1_ospat}
  \vspace{-15pt}
\end{figure}
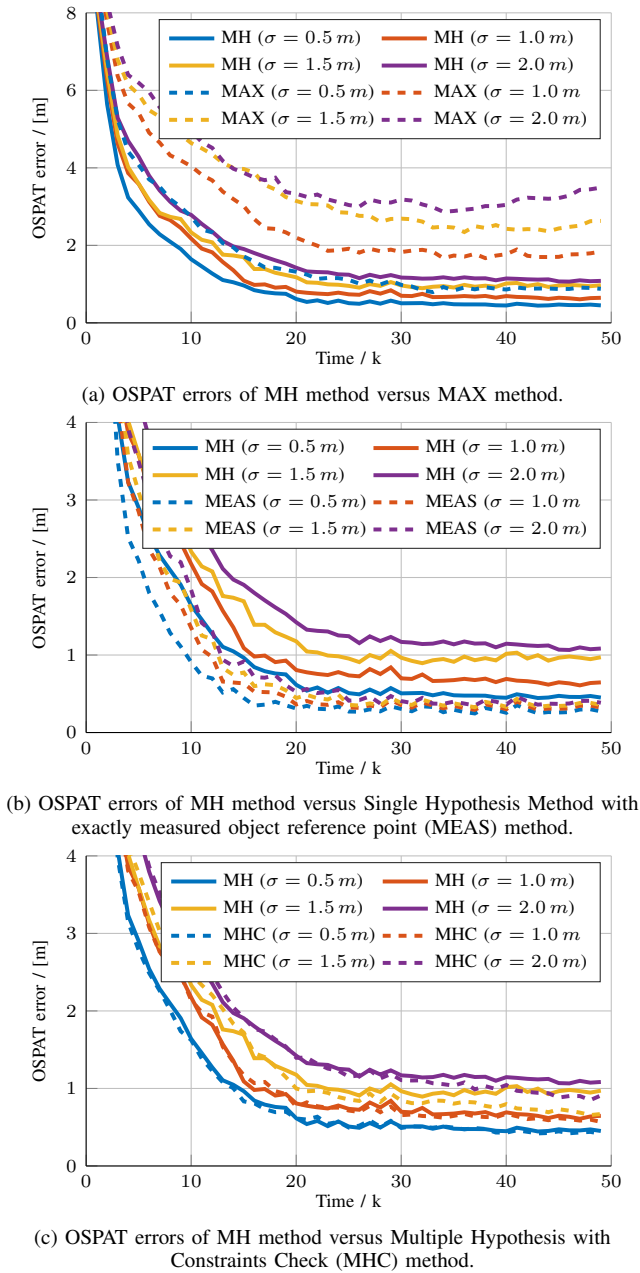%

Some modifications to \cite{env_model_rp}, though, have been made. The adaptive birth process, introduced in \cite{Reuter2014a}, was improved, setting up multiple position and orientation hypotheses of a track and the process model was revised, as well as the parameters of the \ac{UKF} filter. Since the \ac{MH} method requires the incorporation of multiple hypotheses the pruning and merging parameters of the \ac{LMB} filter had to be adjusted, too. Hereby, a trade-off between optimality and computational manageability has been created, which seems to guarantee a fair comparison between all methods. The \ac{LMB} filter prunes components having weights less than $10^{-5}$, hence keeps most of them, and merges every two components having a Bhattacharyya distance closer than $1$. Further, a maximum number of $30$ components is allowed, which is guaranteed by additionally pruning supernumerous components of lowest weights. Finally, the \ac{OSPAT} \cite{ospat} is used to evaluate the performance, where the parameters are also modified in order to highlight the improvements in the cardinality error as well as the localization error. Here, the \ac{OSPAT} order $p = 1$, cut-off $c = 10$ and label change penalty $\alpha = c$ have been chosen.

Fig. \ref{fig:multihyp_vs_maxhyp_ospat} shows a comparison of the \ac{OSPAT} errors of the proposed \ac{MH} method and the 2019 \ac{MAX} method, where it can clearly be seen that the overall performance of \ac{MH} is better. Even in bad situations with heavy measurement noise, where \ac{MAX} mostly fails, \ac{MH} yields good estimation results, yet as good as estimation results of the $\sigma = 0.5$ scenario of \ac{MAX}. More detailed analysis points out that both the localization and the cardinality error are improved, where the number of non-continuous trajectories is reduced significantly by $30\,\%$ to $50\,\%$. Thus, it is of great interest to compare the proposed method with the optimal estimation, where the \ac{LMB} filter exactly knows the measured object reference point, referred to as \acf{MEAS}. Such a comparison can be seen in Fig. \ref{fig:multihyp_vs_measured_ospat2}, where the \ac{MEAS} method outperforms \ac{MH}, as expected. It is even in the worst case superior to \ac{MH} in the most advantageous scenario with $\sigma = 0.5\,\text{m}$. The reasons are manifold but two main aspects are important. First, the number of hypotheses is strictly bounded and therefore, \ac{MH} can not exploit its full power, and second, \ac{MH} always extracts the most probable component of the estimated \ac{LMB} density, which is often an \textit{incorrect} one. More intelligent extraction rules, thus, may improve the proposed method further.

Finally, a possible performance difference between \ac{MH} and the \ac{MHC} is examined in Fig. \ref{fig:multihyp_vs_multihyp_ccheck_ospat2}. As can be seen the constraint check does not degrade performance, it rather leads to a performance increase in some cases, especially when noise is severe. This happens since the most likely association often leads to illegal states, which are rejected, while the \textit{correct} components obtain increased weights and therefore dominate in the extraction, which increases the overall performance.

Due to the specific design of the simulated scenario, with a simple and non-ambiguous situation in the beginning that turns into a very complex one, it can be concluded that the presented method counteracts the mentioned ambiguity of the track to measurement association, whereas it does never decrease precision (which has also been tested in different scenarios). This is achieved only by the cost of computational effort, which increases around $77.9\,\%$ and $102.2\,\%$ between the \ac{MAX} method and the \ac{MH} method. These numbers have been measured on a standard desktop PC and can therefore only roughly indicate the dimension of the increase. The increase of the computational effort correlates to the increase of the measurement noise, which is simply due to the fact that the number of ambiguous situations increases and, with that, the number of relevant mixture components in the update.
Moreover, the model based thinning of \ac{MHC} showed to reduce the computational costs by about $10\,\%$ compared to the \ac{MH} method. Further reductions of the computational effort could be achieved by a parallelization of the innovations and the group updates. Additionally, the approach can be used in combination with Gibbs sampling \cite{MultipleHypothesis2018} for a further reduction. Finally, a dynamic switching between \ac{MAX} and \ac{MHC} based on statistical characteristics (e.g. similar to \cite{ALMB16}) could be implemented.

\subsection{Real Data}
Real data from a test site in Ulm-Lehr \cite{mecview} has been used to proof the applicability of the proposed method to real world scenarios. A complex scenario with multiple vehicles has been chosen, where ground truth information of one test vehicle was available, shown in green in Fig. \ref{fig:real_data_scenario}. The reference vehicle leads a group of three consecutive vehicles and turns left at the intersection, as the third vehicle does. Seven sensors have been distributed around the test site of which two are lidars (A,B) with 16 static beams and five are monocular cameras (C-G) \cite{digital_mirror}, whose positions are also marked in Fig. \ref{fig:real_data_scenario}. All sensors measure the position of the objects in an unknown object reference point and in order to fulfill the requirements, the cameras measure length and width of the vehicles, too.% The sensors' accuracies were evaluated from test vehicle's ground truth data and are summarized in Table~\ref{tab:sensor-accuracy}.
\begin{figure}
  \vspace{4pt}
  %\centering
  \subfloat{
    \setlength\figW{0.85\columnwidth}
    \input{real_data_scenario.tex}}
  \caption{Simulation scenario with three vehicles observed by seven distributed sensors.}
  \label{fig:real_data_scenario}
  \hfill
  \vspace*{-15pt}
\end{figure}
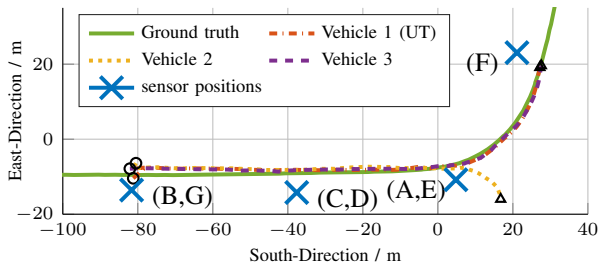%

As can be seen in Fig. \ref{fig:real_data_scenario}, the filter can follow all vehicles as long as those are within the system's range, and the cardinality estimation is correct, too. The average location error between the ground truth track (green) and its estimation (vehicle 1, red) is $\bar{\epsilon}_d = 1.13\,m$ with a maximum of $\epsilon_{d,\text{max}} = 4.01\,m$, which occurs in the track initialization due to the unknown reference point $\zeta$. Again, the track initialization plays a decisive role and superposes the location error. The average error is in the range of the sensor error, which has been evaluated from test vehicle's ground truth data.

\addtolength{\textheight}{-9cm}   % This command serves to balance the column lengths

%%%%%%%%%%%%%%%%%%%%%%%%%%%%%%%%%%%%%%%%%%%%%%%%%%%%%%%%%%%%%%%%%%%%%%%%%%%%%%%%

\section{CONCLUSIONS} \label{conclusions}
Motivated by performance issues of the \ac{LMB} filter of \cite{env_model_rp} in the case of unknown association of measurements and object reference points, an improvement was developed and proposed in this paper. It could be shown that the incorporation of multiple concurrent object reference point hypotheses solves those issues, however, at the cost of increased computational burden. Those, however, have been addressed in a model based thinning that showed the ability to again reduce the increased computational costs without performance degradation. The proposed method has successfully been applied to a real urban and complex scenario with multiple vehicles and sensors, showing good results and improving the underlying measurement model significantly.

                                  % sure that you do not shorten the textheight too much. Standard -12cm

%%%%%%%%%%%%%%%%%%%%%%%%%%%%%%%%%%%%%%%%%%%%%%%%%%%%%%%%%%%%%%%%%%%%%%%%%%%%%%%%

%%%%%%%%%%%%%%%%%%%%%%%%%%%%%%%%%%%%%%%%%%%%%%%%%%%%%%%%%%%%%%%%%%%%%%%%%%%%%%%%

%%%%%%%%%%%%%%%%%%%%%%%%%%%%%%%%%%%%%%%%%%%%%%%%%%%%%%%%%%%%%%%%%%%%%%%%%%%%%%%%

%%%%%%%%%%%%%%%%%%%%%%%%%%%%%%%%%%%%%%%%%%%%%%%%%%%%%%%%%%%%%%%%%%%%%%%%%%%%%%%%

\bibliographystyle{IEEEtran}
\bibliography{IEEEabrv,bibliography}

\end{document}

%% file: assoc_example_false.tex
\definecolor{mycolor1}{rgb}{0.00000,0.44700,0.74100}%
\begin{tikzpicture}

\begin{axis}[%
width=0.951\figW,
height=0.233\figW,
at={(0\figW,0\figW)},
scale only axis,
xmin=-2.5,
xmax=22,
xlabel style={font=\color{white!15!black}},
xlabel={x [m]},
ymin=-2,
ymax=4,
ylabel style={font=\color{white!15!black}},
ylabel={y [m]},
axis background/.style={fill=white},
axis x line*=bottom,
axis y line*=left,
xmajorgrids,
ymajorgrids,
legend style={at={(0.5,0.97)}, anchor=north, legend cell align=left, align=left, draw=white!15!black},
ticklabel style={font={\scriptsize}},legend style={font={\scriptsize}},legend columns={3},xlabel style={font={\scriptsize}},ylabel shift={-3pt},ylabel style={font={\scriptsize}},xlabel shift={-2pt}
]
\addplot [color=green, draw=none, mark size=1.8pt, mark=square, mark options={solid, green}, forget plot]
  table[row sep=crcr]{%
0	0\\
};
\addplot [color=green, forget plot]
  table[row sep=crcr]{%
2	1\\
2	-1\\
2	-1\\
-2	-1\\
-2	1\\
2	1\\
2	0\\
0.666666666666667	1\\
0.666666666666667	-1\\
2	0\\
};
\addplot [color=green, draw=none, mark size=1.8pt, mark=square, mark options={solid, green}, forget plot]
  table[row sep=crcr]{%
5	0\\
};
\addplot [color=red, draw=none, mark size=2.5pt, mark=asterisk, mark options={solid, red}, forget plot]
  table[row sep=crcr]{%
7	-0.1\\
};
\node[right, align=left]
at (axis cs:7.1,-0.1) {$\text{Z}_\text{1}$};
\addplot [color=green, forget plot]
  table[row sep=crcr]{%
7	1\\
7	-1\\
7	-1\\
3	-1\\
3	1\\
7	1\\
7	0\\
5.66666666666667	1\\
5.66666666666667	-1\\
7	0\\
};
\addplot [color=green, draw=none, mark size=1.8pt, mark=square, mark options={solid, green}, forget plot]
  table[row sep=crcr]{%
10	0\\
};
\addplot [color=red, draw=none, mark size=2.5pt, mark=asterisk, mark options={solid, red}, forget plot]
  table[row sep=crcr]{%
12	0.1\\
};
\node[right, align=left]
at (axis cs:12.1,0.1) {$\text{Z}_\text{2}$};
\addplot [color=green, forget plot]
  table[row sep=crcr]{%
12	1\\
12	-1\\
12	-1\\
8	-1\\
8	1\\
12	1\\
12	0\\
10.6666666666667	1\\
10.6666666666667	-1\\
12	0\\
};
\addplot [color=green, draw=none, mark size=1.8pt, mark=square, mark options={solid, green}, forget plot]
  table[row sep=crcr]{%
15	0\\
};
\addplot [color=red, draw=none, mark size=2.5pt, mark=asterisk, mark options={solid, red}, forget plot]
  table[row sep=crcr]{%
17	0.8\\
};
\addplot [color=red, draw=none, mark size=2.5pt, mark=asterisk, mark options={solid, red}]
  table[row sep=crcr]{%
17	0.8\\
};
\addlegendentry{Measurement}

\addplot [color=green, draw=none, mark size=1.8pt, mark=square, mark options={solid, green}]
  table[row sep=crcr]{%
15	0\\
};
\addlegendentry{Ground Truth}

\node[right, align=left]
at (axis cs:17.1,0.8) {$\text{Z}_\text{3}$};
\addplot [color=green, forget plot]
  table[row sep=crcr]{%
17	1\\
17	-1\\
17	-1\\
13	-1\\
13	1\\
17	1\\
17	0\\
15.6666666666667	1\\
15.6666666666667	-1\\
17	0\\
};
\addplot [color=mycolor1, draw=none, mark size=4.3pt, mark=diamond, mark options={solid, mycolor1}, forget plot]
  table[row sep=crcr]{%
4.8927641918761	0.648739836354908\\
};
\addplot [color=mycolor1, forget plot]
  table[row sep=crcr]{%
6.74358824588271	1.87692669748674\\
6.98561699103907	-0.0955449248635855\\
6.98561699103907	-0.0955449248635855\\
3.0419401378695	-0.579447024776928\\
2.79991139271314	1.3930245975734\\
6.74358824588271	1.87692669748674\\
6.86460261846089	0.890690886311579\\
5.42902929482619	1.71562599751563\\
5.67105803998255	-0.2568456248347\\
6.86460261846089	0.890690886311579\\
};
\addplot [color=mycolor1, draw=none, mark size=4.3pt, mark=diamond, mark options={solid, mycolor1}, forget plot]
  table[row sep=crcr]{%
9.42170099252496	0.125732914209633\\
};
\addplot [color=mycolor1, forget plot]
  table[row sep=crcr]{%
11.6252872086283	-0.021508646988622\\
11.1885983733617	-1.19923724950605\\
11.1885983733617	-1.19923724950605\\
7.21811477642164	0.272974475407889\\
7.65480361168818	1.45070307792532\\
11.6252872086283	-0.021508646988622\\
11.406942790995	-0.610372948247338\\
10.3017926763149	0.469228594649359\\
9.86510384104837	-0.708500007868073\\
11.406942790995	-0.610372948247338\\
};
\addplot [color=mycolor1, draw=none, mark size=4.3pt, mark=diamond, mark options={solid, mycolor1}, forget plot]
  table[row sep=crcr]{%
14.8778992999119	-0.040409633910105\\
};
\addplot [color=mycolor1, forget plot]
  table[row sep=crcr]{%
16.6548049939405	0.503254605320594\\
16.7287592710966	0.124763056192855\\
16.7287592710966	0.124763056192855\\
13.1009936058833	-0.584073873140804\\
13.0270393287273	-0.205582324013065\\
16.6548049939405	0.503254605320594\\
16.6917821325186	0.314008830756724\\
15.4455497722028	0.266975628876041\\
15.5195040493588	-0.111515920251698\\
16.6917821325186	0.314008830756724\\
};
\addplot [color=mycolor1, draw=none, mark size=4.3pt, mark=diamond, mark options={solid, mycolor1}]
  table[row sep=crcr]{%
14.8778992999119	-0.040409633910105\\
};
\addlegendentry{State Estimation}

\node[right, align=left, font=\color{mycolor1}]
at (axis cs:12.378,-1.49) {$\text{FR}\rightarrow\text{FL}\rightarrow\text{FL}$};
\end{axis}
\end{tikzpicture}%

%% file: assoc_example_corr.tex
\definecolor{mycolor1}{rgb}{0.00000,0.44700,0.74100}%
\begin{tikzpicture}

\begin{axis}[%
width=0.951\figW,
height=0.233\figW,
at={(0\figW,0\figW)},
scale only axis,
xmin=-2.5,
xmax=22,
xlabel style={font=\color{white!15!black}},
xlabel={x [m]},
ymin=-4,
ymax=2,
ylabel style={font=\color{white!15!black}},
ylabel={y [m]},
axis background/.style={fill=white},
axis x line*=bottom,
axis y line*=left,
xmajorgrids,
ymajorgrids,
ticklabel style={font={\scriptsize}},legend style={font={\scriptsize}},legend columns={3},xlabel style={font={\scriptsize}},ylabel shift={-3pt},ylabel style={font={\scriptsize}},xlabel shift={-2pt}
]
\addplot [color=green, draw=none, mark size=1.8pt, mark=square, mark options={solid, green}, forget plot]
  table[row sep=crcr]{%
0	0\\
};
\addplot [color=green, forget plot]
  table[row sep=crcr]{%
2	1\\
2	-1\\
2	-1\\
-2	-1\\
-2	1\\
2	1\\
2	0\\
0.666666666666667	1\\
0.666666666666667	-1\\
2	0\\
};
\addplot [color=green, draw=none, mark size=1.8pt, mark=square, mark options={solid, green}, forget plot]
  table[row sep=crcr]{%
5	0\\
};
\addplot [color=red, draw=none, mark size=2.5pt, mark=asterisk, mark options={solid, red}, forget plot]
  table[row sep=crcr]{%
7	-0.1\\
};
\node[right, align=left]
at (axis cs:7.1,-0.1) {$\text{Z}_\text{1}$};
\addplot [color=green, forget plot]
  table[row sep=crcr]{%
7	1\\
7	-1\\
7	-1\\
3	-1\\
3	1\\
7	1\\
7	0\\
5.66666666666667	1\\
5.66666666666667	-1\\
7	0\\
};
\addplot [color=green, draw=none, mark size=1.8pt, mark=square, mark options={solid, green}, forget plot]
  table[row sep=crcr]{%
10	0\\
};
\addplot [color=red, draw=none, mark size=2.5pt, mark=asterisk, mark options={solid, red}, forget plot]
  table[row sep=crcr]{%
12	0.1\\
};
\node[right, align=left]
at (axis cs:12.1,0.1) {$\text{Z}_\text{2}$};
\addplot [color=green, forget plot]
  table[row sep=crcr]{%
12	1\\
12	-1\\
12	-1\\
8	-1\\
8	1\\
12	1\\
12	0\\
10.6666666666667	1\\
10.6666666666667	-1\\
12	0\\
};
\addplot [color=green, draw=none, mark size=1.8pt, mark=square, mark options={solid, green}, forget plot]
  table[row sep=crcr]{%
15	0\\
};
\addplot [color=red, draw=none, mark size=2.5pt, mark=asterisk, mark options={solid, red}, forget plot]
  table[row sep=crcr]{%
17	0.8\\
};
\addplot [color=red, draw=none, mark size=2.5pt, mark=asterisk, mark options={solid, red}, forget plot]
  table[row sep=crcr]{%
17	0.8\\
};
\addplot [color=green, draw=none, mark size=1.8pt, mark=square, mark options={solid, green}, forget plot]
  table[row sep=crcr]{%
15	0\\
};
\node[right, align=left]
at (axis cs:17.1,0.8) {$\text{Z}_\text{3}$};
\addplot [color=green, forget plot]
  table[row sep=crcr]{%
17	1\\
17	-1\\
17	-1\\
13	-1\\
13	1\\
17	1\\
17	0\\
15.6666666666667	1\\
15.6666666666667	-1\\
17	0\\
};
\addplot [color=mycolor1, draw=none, mark size=4.3pt, mark=diamond, mark options={solid, mycolor1}, forget plot]
  table[row sep=crcr]{%
4.86893401229302	-0.792904244433776\\
};
\addplot [color=mycolor1, forget plot]
  table[row sep=crcr]{%
6.97805872970701	-0.106626709628965\\
6.6830297160281	-2.06900796516902\\
6.6830297160281	-2.06900796516902\\
2.75980929487902	-1.47918177923859\\
3.05483830855793	0.483199476301473\\
6.97805872970701	-0.106626709628965\\
6.83054422286756	-1.08781733739899\\
5.67031858932399	0.0899820190145143\\
5.37528957564507	-1.87239923652555\\
6.83054422286756	-1.08781733739899\\
};
\addplot [color=mycolor1, draw=none, mark size=4.3pt, mark=diamond, mark options={solid, mycolor1}, forget plot]
  table[row sep=crcr]{%
10.087492971877	-1.18434102177905\\
};
\addplot [color=mycolor1, forget plot]
  table[row sep=crcr]{%
11.9170238985581	0.0634244791798326\\
12.1868597726863	-1.88915389958809\\
12.1868597726863	-1.88915389958809\\
8.25796204519588	-2.43210652273793\\
7.98812617106767	-0.479528143970013\\
11.9170238985581	0.0634244791798326\\
12.0519418356222	-0.912864710204127\\
10.6073913227279	-0.117559728536782\\
10.8772271968562	-2.0701381073047\\
12.0519418356222	-0.912864710204127\\
};
\addplot [color=mycolor1, draw=none, mark size=4.3pt, mark=diamond, mark options={solid, mycolor1}, forget plot]
  table[row sep=crcr]{%
15.2294894994294	-0.468391610258749\\
};
\addplot [color=mycolor1, forget plot]
  table[row sep=crcr]{%
16.9647448528265	0.792724677840022\\
17.3298638111753	-0.904227799851998\\
17.3298638111753	-0.904227799851998\\
13.4942341460323	-1.72950789835752\\
13.1291151876834	-0.0325554206654993\\
16.9647448528265	0.792724677840022\\
17.1473043320009	-0.0557515610059879\\
15.6862016311121	0.517631311671515\\
16.0513205894609	-1.1793211660205\\
17.1473043320009	-0.0557515610059879\\
};
\addplot [color=mycolor1, draw=none, mark size=4.3pt, mark=diamond, mark options={solid, mycolor1}, forget plot]
  table[row sep=crcr]{%
15.2294894994294	-0.468391610258749\\
};
\node[right, align=left, font=\color{mycolor1}]
at (axis cs:12.529,-2.668) {$\text{FL}\rightarrow\text{FL}\rightarrow\text{FL}$};
\end{axis}
\end{tikzpicture}%

%% file: scenario_1.tex
\definecolor{mycolor1}{rgb}{0.00000,0.44700,0.74100}%
\definecolor{mycolor2}{rgb}{0.85000,0.32500,0.09800}%
\definecolor{mycolor3}{rgb}{0.92900,0.69400,0.12500}%
\definecolor{mycolor4}{rgb}{0.49400,0.18400,0.55600}%
\begin{tikzpicture}

\begin{axis}[%
width=0.951\figW,
height=0.367\figW,
at={(0\figW,0\figW)},
scale only axis,
xmin=-33,
xmax=37,
xlabel style={font=\color{white!15!black}},
xlabel={X-Coordinate / m},
ymin=-12,
ymax=15,
ylabel style={font=\color{white!15!black}},
ylabel={Y-Coordinate / m},
axis background/.style={fill=white},
axis x line*=bottom,
axis y line*=left,
xmajorgrids,
ymajorgrids,
legend style={legend cell align=left, align=left, draw=white!15!black},
ticklabel style={font={\scriptsize}},legend style={font={\scriptsize}},legend columns={2},xlabel style={font={\scriptsize}},xlabel shift={-4pt},ylabel style={font={\scriptsize}},ylabel shift={-9pt}
]
\addplot [color=mycolor1, line width=3.5pt]
  table[row sep=crcr]{%
-21	2.25\\
-20.5	2.25\\
-20	2.25\\
-19.5	2.25\\
-19	2.25\\
-18.5	2.25\\
-18	2.25\\
-17.5	2.25\\
-17	2.25\\
-16.5	2.25\\
-16	2.25\\
-15.4950000000001	2.25\\
-14.9950000000001	2.25\\
-14.4950000000002	2.25\\
-13.9950000000002	2.25\\
-13.4950000000003	2.25\\
-12.9950000000003	2.25\\
-12.4950000000004	2.25\\
-11.9950000000004	2.25\\
-11.4950000000005	2.25\\
-10.9950000000005	2.25\\
-10.4950000000006	2.25\\
-9.99500000000066	2.25\\
-9.49500000000071	2.25\\
-8.99500000000077	2.25\\
-8.49500000000082	2.25\\
-7.99500000000088	2.25\\
-7.49500000000093	2.25\\
-6.99500000000099	2.25\\
-6.49500000000104	2.25\\
-5.9950000000011	2.25\\
-5.49500000000115	2.25\\
-4.99500000000121	2.25\\
-4.49500000000126	2.25\\
-3.99500000000132	2.25\\
-3.49500000000137	2.25\\
-2.99500000000143	2.25\\
-2.49500000000148	2.25\\
-1.99500000000154	2.25\\
-1.49500000000159	2.25\\
-0.995000000001646	2.25\\
-0.495000000001479	2.25\\
0.00499999999868805	2.25\\
0.504999999998855	2.25\\
1.00499999999902	2.25\\
1.50499999999919	2.25\\
2.00499999999936	2.25\\
2.50499999999952	2.25\\
3.00499999999969	2.25\\
3.50499999999986	2.25\\
};
\addlegendentry{Vehicle 1}

\addplot [color=black, line width=0.8pt, draw=none, mark=o, mark options={solid, black}, forget plot]
  table[row sep=crcr]{%
-21	2.25\\
};
\addplot [color=black, line width=0.8pt, draw=none, mark=triangle, mark options={solid, black}, forget plot]
  table[row sep=crcr]{%
3.50499999999986	2.25\\
};
\addplot [color=mycolor2, line width=3.5pt]
  table[row sep=crcr]{%
11	4.25\\
10.6000000702476	4.24976293857796\\
10.2000001404953	4.24952587715592\\
9.80000021074296	4.24928881573388\\
9.40000028099061	4.24905175431184\\
9.00000035123827	4.2488146928898\\
8.60000042148592	4.24857763146776\\
8.20000049173357	4.24834057004572\\
7.80000056198122	4.24810350862368\\
7.40000063222888	4.24786644720164\\
7.00000070247653	4.2476293857796\\
6.59600077342671	4.24738995374334\\
6.1960008436744	4.2471528923213\\
5.7960009139221	4.24691583089926\\
5.3960009841698	4.24667876947722\\
4.9960010544175	4.24644170805518\\
4.59600112466519	4.24620464663314\\
4.19600119491289	4.2459675852111\\
3.79600126516059	4.24573052378906\\
3.39600133540829	4.24549346236702\\
2.99600140565598	4.24525640094498\\
2.59600147590368	4.24501933952294\\
2.19600154615138	4.2447822781009\\
1.79600161639907	4.24454521667886\\
1.39600168664677	4.24430815525683\\
0.99600175689447	4.24407109383479\\
0.596001827142167	4.24383403241275\\
0.196001897389865	4.24359697099071\\
-0.203998032362438	4.24335990956867\\
-0.60399796211474	4.24312284814663\\
-1.00399789186704	4.24288578672459\\
-1.40399782161934	4.24264872530255\\
-1.80399775137165	4.24241166388051\\
-2.20399768112395	4.24217460245847\\
-2.60399761087625	4.24193754103643\\
-3.00399754062855	4.24170047961439\\
-3.40399747038086	4.24146341819235\\
-3.80399740013316	4.24122635677031\\
-4.20399732988546	4.24098929534827\\
-4.60399725963777	4.24075223392623\\
-5.00399718939007	4.24051517250419\\
-5.40399711914255	4.24027811108215\\
-5.80399704889503	4.24004104966011\\
-6.20399697864751	4.23980398823807\\
-6.60399690839999	4.23956692681603\\
-7.00399683815247	4.23932986539399\\
-7.40399676790495	4.23909280397195\\
-7.80399669765743	4.23885574254991\\
-8.20399662740991	4.23861868112787\\
-8.60399655716239	4.23838161970583\\
};
\addlegendentry{Vehicle 2}

\addplot [color=black, line width=0.8pt, draw=none, mark=o, mark options={solid, black}, forget plot]
  table[row sep=crcr]{%
11	4.25\\
};
\addplot [color=black, line width=0.8pt, draw=none, mark=triangle, mark options={solid, black}, forget plot]
  table[row sep=crcr]{%
-8.60399655716239	4.23838161970583\\
};
\addplot [color=mycolor3, line width=3.5pt]
  table[row sep=crcr]{%
-15	-7\\
-14.5386748626225	-6.57834567894669\\
-14.0694220418397	-6.16553211445693\\
-13.592410932824	-5.76170832774437\\
-13.1078137314052	-5.36702009481049\\
-12.6158053719105	-4.9816098938211\\
-12.1165634640145	-4.60561685367321\\
-11.6102682286241	-4.23917670377106\\
-11.097102432821	-3.88242172502923\\
-10.5772513238845	-3.5354807021207\\
-10.0509025624194	-3.198478876987\\
-9.51288838563856	-2.86831970358603\\
-8.97405643959467	-2.55165997649433\\
-8.42930557696477	-2.24529459577515\\
-7.87883244689429	-1.9493341560038\\
-7.32283576420508	-1.66388549568494\\
-6.7615162376615	-1.38905165868498\\
-6.19507649751672	-1.12493185703439\\
-5.62372102236541	-0.871621435113245\\
-5.04765606532923	-0.629211835232915\\
-4.46708957960163	-0.397790564626403\\
-3.88223114337901	-0.177441163859168\\
-3.29329188420534	0.0317568233280914\\
-2.70048440275731	0.229727878733882\\
-2.10402269609787	0.416400536957159\\
-1.50412208042558	0.591707411195562\\
-0.90099911334779	0.755585217571313\\
-0.294871515705745	0.907974797976017\\
0.314041907020324	1.04882114142609\\
0.925521343697597	1.17807340392113\\
1.53934605689006	1.29568492679805\\
2.15529446254236	1.40161325357435\\
2.77314420996922	1.49582014527444\\
3.39267226212166	1.5782715942335\\
4.01365497610088	1.64893783637387\\
4.63586818389098	1.70779336194958\\
5.25908727328115	1.75481692475504\\
5.88308726894826	1.78999154979478\\
6.50764291367047	1.81330453941114\\
7.1325287496427	1.82474747786805\\
7.75751919986438	1.824316234389\\
8.38238864957064	1.81201096464819\\
9.00691152767588	1.78783611071435\\
9.63086238820342	1.75180039944724\\
10.2540159916687	1.70391683934726\\
10.8761473863882	1.64420271585959\\
11.4970319896849	1.57267958513432\\
12.1164456689596	1.48937326624492\\
12.7341648226011	1.39431383186779\\
13.3499664607031	1.2875355974264\\
};
\addlegendentry{Vehicle 3}

\addplot [color=black, line width=0.8pt, draw=none, mark=o, mark options={solid, black}, forget plot]
  table[row sep=crcr]{%
-15	-7\\
};
\addplot [color=black, line width=0.8pt, draw=none, mark=triangle, mark options={solid, black}, forget plot]
  table[row sep=crcr]{%
13.3499664607031	1.2875355974264\\
};
\addplot [color=mycolor4, line width=1.5pt, draw=none, mark size=6.0pt, mark=x, mark options={solid, mycolor4}]
  table[row sep=crcr]{%
-30	10\\
-30	-10\\
30	-10\\
};
\addlegendentry{sensor positions}

\node[right, align=left]
at (axis cs:-29,11) {(A)};
\node[right, align=left]
at (axis cs:-29,-9) {(B)};
\node[right, align=left]
at (axis cs:31,-9) {(C)};
\end{axis}
\end{tikzpicture}%

%% file: multihyp_vs_maxhyp_ospat.tex
\definecolor{mycolor1}{rgb}{0.00000,0.44700,0.74100}%
\definecolor{mycolor2}{rgb}{0.85000,0.32500,0.09800}%
\definecolor{mycolor3}{rgb}{0.92900,0.69400,0.12500}%
\definecolor{mycolor4}{rgb}{0.49400,0.18400,0.55600}%
\begin{tikzpicture}

\begin{axis}[%
width=0.951\figW,
height=0.562\figW,
at={(0\figW,0\figW)},
scale only axis,
xmin=0,
xmax=50,
xlabel style={font=\color{white!15!black}},
xlabel={Time / k},
ymin=0,
ymax=8,
ylabel style={font=\color{white!15!black}},
ylabel={OSPAT error / [m]},
axis background/.style={fill=white},
axis x line*=bottom,
axis y line*=left,
xmajorgrids,
ymajorgrids,
legend style={legend cell align=left, align=left, draw=white!15!black},
ticklabel style={font={\scriptsize}},legend style={font={\scriptsize}},legend columns={2},xlabel style={font={\scriptsize}},xlabel shift={-4pt},ylabel style={font={\scriptsize}}
]
\addplot [color=mycolor1, line width=1.5pt]
  table[row sep=crcr]{%
0	10\\
1	8.15234838853789\\
2	5.63937423263064\\
3	4.07775258527992\\
4	3.23374275067016\\
5	2.90051727953734\\
6	2.54270549598684\\
7	2.28554441261806\\
8	2.09467862961786\\
9	1.90749489136874\\
10	1.64093966815805\\
11.00999	1.45702491614759\\
12.00999	1.27222086490484\\
13.00999	1.10540465443589\\
14.00999	1.0479382762854\\
15.00999	0.964698734328788\\
16.00999	0.843805182400804\\
17.00999	0.791346831523113\\
18.00999	0.761662372384978\\
19.00999	0.749899361833445\\
20.00999	0.614909489970466\\
21.00999	0.540728346153793\\
22.00999	0.581236150548764\\
23.00999	0.513525053908324\\
24.00999	0.574077862596073\\
25.00999	0.50833917192089\\
26.00999	0.492789162516431\\
27.00999	0.554461061892814\\
28.00999	0.495623141336383\\
29.00999	0.573522967194084\\
30.00999	0.50519801155217\\
31.00999	0.507960692194921\\
32.00999	0.511092176326815\\
33.00999	0.480432617429218\\
34.00999	0.508376247378201\\
35.00999	0.480668251472092\\
36.00999	0.476459950888544\\
37.00999	0.476424301881209\\
38.00999	0.475172909055577\\
39.00999	0.454199660248637\\
40.00999	0.448738511469451\\
41.00999	0.493491579346936\\
42.00999	0.471760550293187\\
43.00999	0.445179477221264\\
44.00999	0.451891727738896\\
45.00999	0.473082747251186\\
46.00999	0.459766573818695\\
47.00999	0.458644324429796\\
48.00999	0.47145745626503\\
49.00999	0.451256014107812\\
};
\addlegendentry{MH ($\sigma = 0.5\,m$)}

\addplot [color=mycolor2, line width=1.5pt]
  table[row sep=crcr]{%
0	10\\
1	8.32743053393112\\
2	6.26884608199765\\
3	4.62586744146566\\
4	3.86538230823443\\
5	3.58412203540523\\
6	3.12484233541889\\
7	2.76153248554676\\
8	2.65171229368366\\
9	2.45284158549561\\
10	2.17536328697398\\
11.00999	1.91595615083685\\
12.00999	1.82033359862364\\
13.00999	1.57127731285294\\
14.00999	1.35841637724789\\
15.00999	1.09052350182158\\
16.00999	0.980260654056447\\
17.00999	0.992958424603424\\
18.00999	0.908400430565633\\
19.00999	0.928170452416071\\
20.00999	0.806215409476017\\
21.00999	0.77674650752127\\
22.00999	0.756175680459313\\
23.00999	0.741941025948544\\
24.00999	0.779524430099377\\
25.00999	0.736705350530727\\
26.00999	0.725009857103726\\
27.00999	0.818660290661454\\
28.00999	0.754355280141109\\
29.00999	0.840287034461273\\
30.00999	0.698834047207763\\
31.00999	0.691565938040006\\
32.00999	0.748772796242274\\
33.00999	0.664914920334234\\
34.00999	0.668692851724568\\
35.00999	0.676278877357639\\
36.00999	0.692662568969323\\
37.00999	0.676901856455699\\
38.00999	0.695290171009879\\
39.00999	0.650509973174577\\
40.00999	0.691567613832743\\
41.00999	0.669632310621276\\
42.00999	0.662232045057157\\
43.00999	0.646393252353671\\
44.00999	0.61801528173026\\
45.00999	0.648148924703421\\
46.00999	0.634064370791283\\
47.00999	0.608626392240478\\
48.00999	0.638703154201612\\
49.00999	0.647077955527055\\
};
\addlegendentry{MH ($\sigma = 1.0\,m$)}

\addplot [color=mycolor3, line width=1.5pt]
  table[row sep=crcr]{%
0	10\\
1	8.45964379342644\\
2	6.48916806144404\\
3	4.83971090218149\\
4	4.01967189022007\\
5	3.58947245530542\\
6	3.16927698588751\\
7	2.83641803542556\\
8	2.73548549248234\\
9	2.68079917632438\\
10	2.33365401705688\\
11.00999	2.14340791608802\\
12.00999	2.07826404062335\\
13.00999	1.76079413739608\\
14.00999	1.73910403252804\\
15.00999	1.69358833250982\\
16.00999	1.39000633294247\\
17.00999	1.39005999446777\\
18.00999	1.29242636191549\\
19.00999	1.23597237044657\\
20.00999	1.17618255029825\\
21.00999	1.04567553579414\\
22.00999	1.03362901040116\\
23.00999	0.993496913044519\\
24.00999	0.983596193156641\\
25.00999	0.95404837118323\\
26.00999	0.908676901356342\\
27.00999	0.999846276763486\\
28.00999	0.954898693136438\\
29.00999	1.05313044724341\\
30.00999	0.9650171834698\\
31.00999	0.924483779964317\\
32.00999	0.894567166076881\\
33.00999	0.932045339234012\\
34.00999	0.948147933119518\\
35.00999	0.912476329683331\\
36.00999	0.979813067410036\\
37.00999	0.933023537897898\\
38.00999	0.959560069527999\\
39.00999	0.930622460147041\\
40.00999	1.01611113412896\\
41.00999	1.02978375897882\\
42.00999	0.954766586967735\\
43.00999	0.997816624514943\\
44.00999	0.927659446388562\\
45.00999	0.965300131505333\\
46.00999	0.980982768585791\\
47.00999	0.965939450973383\\
48.00999	0.946759877636097\\
49.00999	0.971364317573702\\
};
\addlegendentry{MH ($\sigma = 1.5\,m$)}

\addplot [color=mycolor4, line width=1.5pt]
  table[row sep=crcr]{%
0	10\\
1	8.66031437052369\\
2	6.86167007732462\\
3	5.29057213785325\\
4	4.69408889743197\\
5	4.28969169302897\\
6	3.79617929514222\\
7	3.41813862926426\\
8	3.12718502527695\\
9	2.8958073097385\\
10	2.78619976485208\\
11.00999	2.55420330684631\\
12.00999	2.35140755665467\\
13.00999	2.11314614788506\\
14.00999	1.97378048140227\\
15.00999	1.90523525239319\\
16.00999	1.80179740854344\\
17.00999	1.70525582600889\\
18.00999	1.61456243731158\\
19.00999	1.53075989240821\\
20.00999	1.42190195991826\\
21.00999	1.32710973563092\\
22.00999	1.31191111853325\\
23.00999	1.30364645083912\\
24.00999	1.25311584853998\\
25.00999	1.24275276403615\\
26.00999	1.15365161517747\\
27.00999	1.24094945962329\\
28.00999	1.18034024651067\\
29.00999	1.23309455556241\\
30.00999	1.16884381616764\\
31.00999	1.16031182063282\\
32.00999	1.14170291844947\\
33.00999	1.14680760411208\\
34.00999	1.17539428625077\\
35.00999	1.13779297050882\\
36.00999	1.17446996969669\\
37.00999	1.14768403898529\\
38.00999	1.15073885688751\\
39.00999	1.09946629167165\\
40.00999	1.14595237585168\\
41.00999	1.13634200793792\\
42.00999	1.12407401179165\\
43.00999	1.11692599861025\\
44.00999	1.1228878784843\\
45.00999	1.06910665593809\\
46.00999	1.06514114889028\\
47.00999	1.09815695190929\\
48.00999	1.07489857978567\\
49.00999	1.08377325289822\\
};
\addlegendentry{MH ($\sigma = 2.0\,m$)}

\addplot [color=mycolor1, dashed, line width=1.5pt]
  table[row sep=crcr]{%
0	10\\
1	9.06561085691992\\
2	6.79758639775544\\
3	5.24800898542096\\
4	4.42117134023403\\
5	4.07674740132892\\
6	3.70111151297489\\
7	3.51286708376048\\
8	3.19406940383722\\
9	2.99263417867098\\
10	2.73886286086189\\
11.00999	2.40816318956538\\
12.00999	2.33580657322287\\
13.00999	2.05710092828217\\
14.00999	1.94905448809576\\
15.00999	1.71518057871895\\
16.00999	1.58370528474524\\
17.00999	1.43375009928754\\
18.00999	1.4195707045792\\
19.00999	1.36703461488333\\
20.00999	1.31963697595765\\
21.00999	1.24661660922746\\
22.00999	1.13479056692712\\
23.00999	1.1586734920184\\
24.00999	1.15437244688523\\
25.00999	1.10869437356372\\
26.00999	1.01314541606758\\
27.00999	1.09199260979147\\
28.00999	0.990714614930888\\
29.00999	1.08657331763521\\
30.00999	0.9797382108315\\
31.00999	0.939441368778653\\
32.00999	0.842091611328809\\
33.00999	0.793913889244968\\
34.00999	0.926017624686977\\
35.00999	0.835503500334096\\
36.00999	0.913563489232361\\
37.00999	0.865908179559033\\
38.00999	0.91967377802702\\
39.00999	0.881988124969224\\
40.00999	0.861189744247105\\
41.00999	0.913497239485357\\
42.00999	0.881535252937724\\
43.00999	0.886132211865144\\
44.00999	0.886892579244601\\
45.00999	0.874675576698284\\
46.00999	0.875486387936994\\
47.00999	0.910263993802641\\
48.00999	0.900798737490177\\
49.00999	0.878613564237427\\
};
\addlegendentry{MAX ($\sigma = 0.5\,m$)}

\addplot [color=mycolor2, dashed, line width=1.5pt]
  table[row sep=crcr]{%
0	10\\
1	9.34009288628184\\
2	7.75110181860776\\
3	6.3698654546293\\
4	5.66133435518297\\
5	5.41608364372795\\
6	4.90568260275\\
7	4.65094718692605\\
8	4.32271619209023\\
9	4.19428744928957\\
10	4.04242652786772\\
11.00999	3.82062016835883\\
12.00999	3.67656731638416\\
13.00999	3.38689067686751\\
14.00999	3.1955192642314\\
15.00999	3.01288006181897\\
16.00999	2.67820539764942\\
17.00999	2.44356491191067\\
18.00999	2.42515392311696\\
19.00999	2.33554338983284\\
20.00999	2.19706698994308\\
21.00999	2.08993270384731\\
22.00999	1.99276594063915\\
23.00999	1.85942262054353\\
24.00999	1.8750874477354\\
25.00999	1.91235145066094\\
26.00999	1.81255177172837\\
27.00999	1.911242240168\\
28.00999	1.84928262276371\\
29.00999	1.88963624753057\\
30.00999	1.82701605826788\\
31.00999	1.86450669597825\\
32.00999	1.88456171393714\\
33.00999	1.83749253362025\\
34.00999	1.67053207982879\\
35.00999	1.73751306481664\\
36.00999	1.71030044814541\\
37.00999	1.7803322430355\\
38.00999	1.67060039915202\\
39.00999	1.73545876201383\\
40.00999	1.77262356597546\\
41.00999	1.89259360402829\\
42.00999	1.73119516510771\\
43.00999	1.70001537217803\\
44.00999	1.69849898077436\\
45.00999	1.70420606574398\\
46.00999	1.76603194256701\\
47.00999	1.77233055563524\\
48.00999	1.7947688696319\\
49.00999	1.83969021751317\\
};
\addlegendentry{MAX ($\sigma = 1.0\,m$}

\addplot [color=mycolor3, dashed, line width=1.5pt]
  table[row sep=crcr]{%
0	10\\
1	9.43357810841604\\
2	8.05208204292967\\
3	6.78384533271397\\
4	6.14501599966807\\
5	5.89557765223419\\
6	5.48732587042382\\
7	5.31811831993822\\
8	4.9889750372384\\
9	4.85586965572406\\
10	4.64015632265023\\
11.00999	4.50311627028084\\
12.00999	4.3474213484576\\
13.00999	4.1273126391507\\
14.00999	3.99703358999296\\
15.00999	3.96983144250712\\
16.00999	3.87554235938497\\
17.00999	3.56059781585703\\
18.00999	3.44118084911601\\
19.00999	3.2474971565133\\
20.00999	3.14528138348391\\
21.00999	3.06564923091423\\
22.00999	3.05793581195329\\
23.00999	3.02843685784377\\
24.00999	2.85530076335567\\
25.00999	2.84031058557436\\
26.00999	2.76176999774223\\
27.00999	2.75774842448893\\
28.00999	2.60474974703523\\
29.00999	2.68128759328156\\
30.00999	2.6891514539158\\
31.00999	2.67436423827309\\
32.00999	2.6138148811093\\
33.00999	2.4597968503372\\
34.00999	2.48253102412757\\
35.00999	2.4581659323704\\
36.00999	2.34353950895486\\
37.00999	2.49541101159389\\
38.00999	2.44876076074359\\
39.00999	2.52381164543409\\
40.00999	2.42672242906699\\
41.00999	2.40367139704179\\
42.00999	2.41918242263285\\
43.00999	2.41191540594827\\
44.00999	2.36353729053415\\
45.00999	2.43689049713107\\
46.00999	2.47620218785072\\
47.00999	2.50968545186916\\
48.00999	2.62409015512602\\
49.00999	2.63105857595738\\
};
\addlegendentry{MAX ($\sigma = 1.5\,m$)}

\addplot [color=mycolor4, dashed, line width=1.5pt]
  table[row sep=crcr]{%
0	10\\
1	9.49872406851692\\
2	8.268674082309\\
3	7.14493115883078\\
4	6.39730177659894\\
5	6.19488916674599\\
6	5.95543050761017\\
7	5.56537626483926\\
8	5.33136230573748\\
9	5.22121610149031\\
10	4.95295704263408\\
11.00999	4.66551818671299\\
12.00999	4.51449570315854\\
13.00999	4.18043133767095\\
14.00999	3.97866069003912\\
15.00999	3.87163054875783\\
16.00999	3.74922417139819\\
17.00999	3.6822237771254\\
18.00999	3.74590513700973\\
19.00999	3.37426959123368\\
20.00999	3.32930314502806\\
21.00999	3.25191098476376\\
22.00999	3.31177542351525\\
23.00999	3.18541281812563\\
24.00999	3.11618346144506\\
25.00999	3.04022255691361\\
26.00999	2.9998600295765\\
27.00999	3.18134444230513\\
28.00999	3.11756899950347\\
29.00999	3.10530387010847\\
30.00999	3.18421578210252\\
31.00999	3.15628350833605\\
32.00999	3.05610314782131\\
33.00999	2.9851460699312\\
34.00999	2.86687378624076\\
35.00999	2.88690936796448\\
36.00999	2.91888064757486\\
37.00999	2.94724079462793\\
38.00999	3.01377290734472\\
39.00999	3.01234518700203\\
40.00999	3.0458264796415\\
41.00999	3.1812625773982\\
42.00999	3.22957156806963\\
43.00999	3.21593263163635\\
44.00999	3.20256989531726\\
45.00999	3.22074827999819\\
46.00999	3.28437410276275\\
47.00999	3.43180704217998\\
48.00999	3.44688737107456\\
49.00999	3.49004638932394\\
};
\addlegendentry{MAX ($\sigma = 2.0\,m$)}

\end{axis}
\end{tikzpicture}%

%% file: multihyp_vs_measured_ospat.tex
\definecolor{mycolor1}{rgb}{0.00000,0.44700,0.74100}%
\definecolor{mycolor2}{rgb}{0.85000,0.32500,0.09800}%
\definecolor{mycolor3}{rgb}{0.92900,0.69400,0.12500}%
\definecolor{mycolor4}{rgb}{0.49400,0.18400,0.55600}%
\begin{tikzpicture}

\begin{axis}[%
width=0.951\figW,
height=0.562\figW,
at={(0\figW,0\figW)},
scale only axis,
xmin=0,
xmax=50,
xlabel style={font=\color{white!15!black}},
xlabel={Time / k},
ymin=0,
ymax=4,
ylabel style={font=\color{white!15!black}},
ylabel={OSPAT error / [m]},
axis background/.style={fill=white},
axis x line*=bottom,
axis y line*=left,
xmajorgrids,
ymajorgrids,
legend style={legend cell align=left, align=left, draw=white!15!black},
ticklabel style={font={\scriptsize}},legend style={font={\scriptsize}},legend columns={2},xlabel style={font={\scriptsize}},xlabel shift={-4pt},ylabel style={font={\scriptsize}}
]
\addplot [color=mycolor1, line width=1.5pt]
  table[row sep=crcr]{%
0	10\\
1	8.15234838853789\\
2	5.63937423263064\\
3	4.07775258527992\\
4	3.23374275067016\\
5	2.90051727953734\\
6	2.54270549598684\\
7	2.28554441261806\\
8	2.09467862961786\\
9	1.90749489136874\\
10	1.64093966815805\\
11.00999	1.45702491614759\\
12.00999	1.27222086490484\\
13.00999	1.10540465443589\\
14.00999	1.0479382762854\\
15.00999	0.964698734328788\\
16.00999	0.843805182400804\\
17.00999	0.791346831523113\\
18.00999	0.761662372384978\\
19.00999	0.749899361833445\\
20.00999	0.614909489970466\\
21.00999	0.540728346153793\\
22.00999	0.581236150548764\\
23.00999	0.513525053908324\\
24.00999	0.574077862596073\\
25.00999	0.50833917192089\\
26.00999	0.492789162516431\\
27.00999	0.554461061892814\\
28.00999	0.495623141336383\\
29.00999	0.573522967194084\\
30.00999	0.50519801155217\\
31.00999	0.507960692194921\\
32.00999	0.511092176326815\\
33.00999	0.480432617429218\\
34.00999	0.508376247378201\\
35.00999	0.480668251472092\\
36.00999	0.476459950888544\\
37.00999	0.476424301881209\\
38.00999	0.475172909055577\\
39.00999	0.454199660248637\\
40.00999	0.448738511469451\\
41.00999	0.493491579346936\\
42.00999	0.471760550293187\\
43.00999	0.445179477221264\\
44.00999	0.451891727738896\\
45.00999	0.473082747251186\\
46.00999	0.459766573818695\\
47.00999	0.458644324429796\\
48.00999	0.47145745626503\\
49.00999	0.451256014107812\\
};
\addlegendentry{MH ($\sigma = 0.5\,m$)}

\addplot [color=mycolor2, line width=1.5pt]
  table[row sep=crcr]{%
0	10\\
1	8.32743053393112\\
2	6.26884608199765\\
3	4.62586744146566\\
4	3.86538230823443\\
5	3.58412203540523\\
6	3.12484233541889\\
7	2.76153248554676\\
8	2.65171229368366\\
9	2.45284158549561\\
10	2.17536328697398\\
11.00999	1.91595615083685\\
12.00999	1.82033359862364\\
13.00999	1.57127731285294\\
14.00999	1.35841637724789\\
15.00999	1.09052350182158\\
16.00999	0.980260654056447\\
17.00999	0.992958424603424\\
18.00999	0.908400430565633\\
19.00999	0.928170452416071\\
20.00999	0.806215409476017\\
21.00999	0.77674650752127\\
22.00999	0.756175680459313\\
23.00999	0.741941025948544\\
24.00999	0.779524430099377\\
25.00999	0.736705350530727\\
26.00999	0.725009857103726\\
27.00999	0.818660290661454\\
28.00999	0.754355280141109\\
29.00999	0.840287034461273\\
30.00999	0.698834047207763\\
31.00999	0.691565938040006\\
32.00999	0.748772796242274\\
33.00999	0.664914920334234\\
34.00999	0.668692851724568\\
35.00999	0.676278877357639\\
36.00999	0.692662568969323\\
37.00999	0.676901856455699\\
38.00999	0.695290171009879\\
39.00999	0.650509973174577\\
40.00999	0.691567613832743\\
41.00999	0.669632310621276\\
42.00999	0.662232045057157\\
43.00999	0.646393252353671\\
44.00999	0.61801528173026\\
45.00999	0.648148924703421\\
46.00999	0.634064370791283\\
47.00999	0.608626392240478\\
48.00999	0.638703154201612\\
49.00999	0.647077955527055\\
};
\addlegendentry{MH ($\sigma = 1.0\,m$)}

\addplot [color=mycolor3, line width=1.5pt]
  table[row sep=crcr]{%
0	10\\
1	8.45964379342644\\
2	6.48916806144404\\
3	4.83971090218149\\
4	4.01967189022007\\
5	3.58947245530542\\
6	3.16927698588751\\
7	2.83641803542556\\
8	2.73548549248234\\
9	2.68079917632438\\
10	2.33365401705688\\
11.00999	2.14340791608802\\
12.00999	2.07826404062335\\
13.00999	1.76079413739608\\
14.00999	1.73910403252804\\
15.00999	1.69358833250982\\
16.00999	1.39000633294247\\
17.00999	1.39005999446777\\
18.00999	1.29242636191549\\
19.00999	1.23597237044657\\
20.00999	1.17618255029825\\
21.00999	1.04567553579414\\
22.00999	1.03362901040116\\
23.00999	0.993496913044519\\
24.00999	0.983596193156641\\
25.00999	0.95404837118323\\
26.00999	0.908676901356342\\
27.00999	0.999846276763486\\
28.00999	0.954898693136438\\
29.00999	1.05313044724341\\
30.00999	0.9650171834698\\
31.00999	0.924483779964317\\
32.00999	0.894567166076881\\
33.00999	0.932045339234012\\
34.00999	0.948147933119518\\
35.00999	0.912476329683331\\
36.00999	0.979813067410036\\
37.00999	0.933023537897898\\
38.00999	0.959560069527999\\
39.00999	0.930622460147041\\
40.00999	1.01611113412896\\
41.00999	1.02978375897882\\
42.00999	0.954766586967735\\
43.00999	0.997816624514943\\
44.00999	0.927659446388562\\
45.00999	0.965300131505333\\
46.00999	0.980982768585791\\
47.00999	0.965939450973383\\
48.00999	0.946759877636097\\
49.00999	0.971364317573702\\
};
\addlegendentry{MH ($\sigma = 1.5\,m$)}

\addplot [color=mycolor4, line width=1.5pt]
  table[row sep=crcr]{%
0	10\\
1	8.66031437052369\\
2	6.86167007732462\\
3	5.29057213785325\\
4	4.69408889743197\\
5	4.28969169302897\\
6	3.79617929514222\\
7	3.41813862926426\\
8	3.12718502527695\\
9	2.8958073097385\\
10	2.78619976485208\\
11.00999	2.55420330684631\\
12.00999	2.35140755665467\\
13.00999	2.11314614788506\\
14.00999	1.97378048140227\\
15.00999	1.90523525239319\\
16.00999	1.80179740854344\\
17.00999	1.70525582600889\\
18.00999	1.61456243731158\\
19.00999	1.53075989240821\\
20.00999	1.42190195991826\\
21.00999	1.32710973563092\\
22.00999	1.31191111853325\\
23.00999	1.30364645083912\\
24.00999	1.25311584853998\\
25.00999	1.24275276403615\\
26.00999	1.15365161517747\\
27.00999	1.24094945962329\\
28.00999	1.18034024651067\\
29.00999	1.23309455556241\\
30.00999	1.16884381616764\\
31.00999	1.16031182063282\\
32.00999	1.14170291844947\\
33.00999	1.14680760411208\\
34.00999	1.17539428625077\\
35.00999	1.13779297050882\\
36.00999	1.17446996969669\\
37.00999	1.14768403898529\\
38.00999	1.15073885688751\\
39.00999	1.09946629167165\\
40.00999	1.14595237585168\\
41.00999	1.13634200793792\\
42.00999	1.12407401179165\\
43.00999	1.11692599861025\\
44.00999	1.1228878784843\\
45.00999	1.06910665593809\\
46.00999	1.06514114889028\\
47.00999	1.09815695190929\\
48.00999	1.07489857978567\\
49.00999	1.08377325289822\\
};
\addlegendentry{MH ($\sigma = 2.0\,m$)}

\addplot [color=mycolor1, dashed, line width=1.5pt]
  table[row sep=crcr]{%
0	10\\
1	8.90826322181927\\
2	5.67759878511442\\
3	3.51268146344541\\
4	2.51392311169474\\
5	2.21724896044844\\
6	1.80307487785476\\
7	1.56093803352877\\
8	1.28493099810793\\
9	1.11052460932849\\
10	0.908430864011085\\
11.00999	0.734043291180972\\
12.00999	0.698254440855081\\
13.00999	0.516319530307987\\
14.00999	0.560162993198551\\
15.00999	0.446069474994413\\
16.00999	0.347459362250461\\
17.00999	0.360003672077498\\
18.00999	0.396875024747516\\
19.00999	0.35142724020597\\
20.00999	0.309535340415926\\
21.00999	0.343963118009673\\
22.00999	0.306629017244061\\
23.00999	0.300465519836514\\
24.00999	0.337093735573893\\
25.00999	0.270945619858174\\
26.00999	0.268708447648351\\
27.00999	0.299820497491686\\
28.00999	0.273398899123293\\
29.00999	0.336860827545928\\
30.00999	0.304489773132041\\
31.00999	0.27454720206883\\
32.00999	0.338125868077281\\
33.00999	0.329293313238575\\
34.00999	0.261946295529097\\
35.00999	0.294073117602584\\
36.00999	0.264435532559385\\
37.00999	0.24616826713335\\
38.00999	0.316578027141669\\
39.00999	0.2891336602207\\
40.00999	0.254893986555233\\
41.00999	0.354419029068689\\
42.00999	0.294020194269405\\
43.00999	0.249113742111409\\
44.00999	0.264366745382991\\
45.00999	0.258933910242469\\
46.00999	0.263722927187126\\
47.00999	0.296643670379948\\
48.00999	0.300473593527248\\
49.00999	0.273154438653586\\
};
\addlegendentry{MEAS ($\sigma = 0.5\,m$)}

\addplot [color=mycolor2, dashed, line width=1.5pt]
  table[row sep=crcr]{%
0	10\\
1	9.1578805452999\\
2	6.43146029731256\\
3	4.30897366686779\\
4	3.2255817763631\\
5	2.84999757809837\\
6	2.36024288488593\\
7	2.13383232229923\\
8	1.79739002364996\\
9	1.64508922044632\\
10	1.35135338235081\\
11.00999	1.07046783168826\\
12.00999	0.97371702688421\\
13.00999	0.657776532377337\\
14.00999	0.645346449415481\\
15.00999	0.633893535744827\\
16.00999	0.52731861858882\\
17.00999	0.521675532764425\\
18.00999	0.517327525281501\\
19.00999	0.462470302257802\\
20.00999	0.36219980452062\\
21.00999	0.407173336995334\\
22.00999	0.364752532516144\\
23.00999	0.357480006896931\\
24.00999	0.391664458333932\\
25.00999	0.321706941359582\\
26.00999	0.31895335419049\\
27.00999	0.349955254756754\\
28.00999	0.322588983908079\\
29.00999	0.419657842673894\\
30.00999	0.354301060249277\\
31.00999	0.32085986649771\\
32.00999	0.388229094419807\\
33.00999	0.348344419676368\\
34.00999	0.312481032573452\\
35.00999	0.344606132689015\\
36.00999	0.314363921499867\\
37.00999	0.292285539076531\\
38.00999	0.362779396702961\\
39.00999	0.33767422469143\\
40.00999	0.300068679580461\\
41.00999	0.400479129044833\\
42.00999	0.340908058355741\\
43.00999	0.295191197309662\\
44.00999	0.309240248741601\\
45.00999	0.306189552752481\\
46.00999	0.312320815540566\\
47.00999	0.345268170292699\\
48.00999	0.348280939244959\\
49.00999	0.322901495017267\\
};
\addlegendentry{MEAS ($\sigma = 1.0\,m$}

\addplot [color=mycolor3, dashed, line width=1.5pt]
  table[row sep=crcr]{%
0	10\\
1	9.21188969997286\\
2	6.84484006166149\\
3	4.72659521426681\\
4	3.56331870836068\\
5	3.10179431822548\\
6	2.67562078786593\\
7	2.20944972205657\\
8	1.98178884658736\\
9	1.93563025473014\\
10	1.58457132608885\\
11.00999	1.25138941328395\\
12.00999	1.20306609698361\\
13.00999	0.832599711582499\\
14.00999	0.744313252300068\\
15.00999	0.774434714868253\\
16.00999	0.598113360718489\\
17.00999	0.622288496598607\\
18.00999	0.610753002068748\\
19.00999	0.545167554445819\\
20.00999	0.444734673897282\\
21.00999	0.485161602109172\\
22.00999	0.444021025409934\\
23.00999	0.426065606928934\\
24.00999	0.489990509922159\\
25.00999	0.3568076941068\\
26.00999	0.347893665698497\\
27.00999	0.376027381330615\\
28.00999	0.34945514617731\\
29.00999	0.445853866972947\\
30.00999	0.381478234263029\\
31.00999	0.355232091847594\\
32.00999	0.424224885495068\\
33.00999	0.384076302855523\\
34.00999	0.34682425935767\\
35.00999	0.377181430681767\\
36.00999	0.350620105862631\\
37.00999	0.326137062467792\\
38.00999	0.395299634662242\\
39.00999	0.369694600405671\\
40.00999	0.331400913410661\\
41.00999	0.43411509641436\\
42.00999	0.376007398127166\\
43.00999	0.330969997476955\\
44.00999	0.345581160454438\\
45.00999	0.34311538629062\\
46.00999	0.348089799895981\\
47.00999	0.380968840643816\\
48.00999	0.383171040888258\\
49.00999	0.356355048559186\\
};
\addlegendentry{MEAS ($\sigma = 1.5\,m$)}

\addplot [color=mycolor4, dashed, line width=1.5pt]
  table[row sep=crcr]{%
0	10\\
1	9.2719626288239\\
2	7.02932299291172\\
3	5.09697380964847\\
4	3.91383404772429\\
5	3.49504337276042\\
6	3.02490502046188\\
7	2.55925383517013\\
8	2.27415376819025\\
9	2.19766992746427\\
10	1.83825215598235\\
11.00999	1.41434306622336\\
12.00999	1.30067377310271\\
13.00999	0.939240341944045\\
14.00999	0.8581044002898\\
15.00999	0.916421615360636\\
16.00999	0.715189039646051\\
17.00999	0.707330486971551\\
18.00999	0.739148035220294\\
19.00999	0.611128614478074\\
20.00999	0.523728231911782\\
21.00999	0.561201002667247\\
22.00999	0.509659098583285\\
23.00999	0.522492082171217\\
24.00999	0.548733173651043\\
25.00999	0.410962511404338\\
26.00999	0.418550287813865\\
27.00999	0.465898189476379\\
28.00999	0.374535761490661\\
29.00999	0.470698764815859\\
30.00999	0.405686629062161\\
31.00999	0.380836612305704\\
32.00999	0.447199000761369\\
33.00999	0.37330860045986\\
34.00999	0.370171250242682\\
35.00999	0.398769143430646\\
36.00999	0.373549793864996\\
37.00999	0.349613829143866\\
38.00999	0.420578335160444\\
39.00999	0.398247846175708\\
40.00999	0.360060340185641\\
41.00999	0.461564948563927\\
42.00999	0.403686319142034\\
43.00999	0.359583686871623\\
44.00999	0.372235321512947\\
45.00999	0.369407413462006\\
46.00999	0.374816531457702\\
47.00999	0.406944043233932\\
48.00999	0.409732614434927\\
49.00999	0.383953793236288\\
};
\addlegendentry{MEAS ($\sigma = 2.0\,m$)}

\end{axis}
\end{tikzpicture}%

%% file: multihyp_vs_multihyp_ccheck_ospat.tex
\definecolor{mycolor1}{rgb}{0.00000,0.44700,0.74100}%
\definecolor{mycolor2}{rgb}{0.85000,0.32500,0.09800}%
\definecolor{mycolor3}{rgb}{0.92900,0.69400,0.12500}%
\definecolor{mycolor4}{rgb}{0.49400,0.18400,0.55600}%
\begin{tikzpicture}

\begin{axis}[%
width=0.951\figW,
height=0.562\figW,
at={(0\figW,0\figW)},
scale only axis,
xmin=0,
xmax=50,
xlabel style={font=\color{white!15!black}},
xlabel={Time / k},
ymin=0,
ymax=4,
ylabel style={font=\color{white!15!black}},
ylabel={OSPAT error / [m]},
axis background/.style={fill=white},
axis x line*=bottom,
axis y line*=left,
xmajorgrids,
ymajorgrids,
legend style={legend cell align=left, align=left, draw=white!15!black},
ticklabel style={font={\scriptsize}},legend style={font={\scriptsize}},legend columns={2},xlabel style={font={\scriptsize}},xlabel shift={-4pt},ylabel style={font={\scriptsize}}
]
\addplot [color=mycolor1, line width=1.5pt]
  table[row sep=crcr]{%
0	10\\
1	8.15234838853789\\
2	5.63937423263064\\
3	4.07775258527992\\
4	3.23374275067016\\
5	2.90051727953734\\
6	2.54270549598684\\
7	2.28554441261806\\
8	2.09467862961786\\
9	1.90749489136874\\
10	1.64093966815805\\
11.00999	1.45702491614759\\
12.00999	1.27222086490484\\
13.00999	1.10540465443589\\
14.00999	1.0479382762854\\
15.00999	0.964698734328788\\
16.00999	0.843805182400804\\
17.00999	0.791346831523113\\
18.00999	0.761662372384978\\
19.00999	0.749899361833445\\
20.00999	0.614909489970466\\
21.00999	0.540728346153793\\
22.00999	0.581236150548764\\
23.00999	0.513525053908324\\
24.00999	0.574077862596073\\
25.00999	0.50833917192089\\
26.00999	0.492789162516431\\
27.00999	0.554461061892814\\
28.00999	0.495623141336383\\
29.00999	0.573522967194084\\
30.00999	0.50519801155217\\
31.00999	0.507960692194921\\
32.00999	0.511092176326815\\
33.00999	0.480432617429218\\
34.00999	0.508376247378201\\
35.00999	0.480668251472092\\
36.00999	0.476459950888544\\
37.00999	0.476424301881209\\
38.00999	0.475172909055577\\
39.00999	0.454199660248637\\
40.00999	0.448738511469451\\
41.00999	0.493491579346936\\
42.00999	0.471760550293187\\
43.00999	0.445179477221264\\
44.00999	0.451891727738896\\
45.00999	0.473082747251186\\
46.00999	0.459766573818695\\
47.00999	0.458644324429796\\
48.00999	0.47145745626503\\
49.00999	0.451256014107812\\
};
\addlegendentry{MH ($\sigma = 0.5\,m$)}

\addplot [color=mycolor2, line width=1.5pt]
  table[row sep=crcr]{%
0	10\\
1	8.32743053393112\\
2	6.26884608199765\\
3	4.62586744146566\\
4	3.86538230823443\\
5	3.58412203540523\\
6	3.12484233541889\\
7	2.76153248554676\\
8	2.65171229368366\\
9	2.45284158549561\\
10	2.17536328697398\\
11.00999	1.91595615083685\\
12.00999	1.82033359862364\\
13.00999	1.57127731285294\\
14.00999	1.35841637724789\\
15.00999	1.09052350182158\\
16.00999	0.980260654056447\\
17.00999	0.992958424603424\\
18.00999	0.908400430565633\\
19.00999	0.928170452416071\\
20.00999	0.806215409476017\\
21.00999	0.77674650752127\\
22.00999	0.756175680459313\\
23.00999	0.741941025948544\\
24.00999	0.779524430099377\\
25.00999	0.736705350530727\\
26.00999	0.725009857103726\\
27.00999	0.818660290661454\\
28.00999	0.754355280141109\\
29.00999	0.840287034461273\\
30.00999	0.698834047207763\\
31.00999	0.691565938040006\\
32.00999	0.748772796242274\\
33.00999	0.664914920334234\\
34.00999	0.668692851724568\\
35.00999	0.676278877357639\\
36.00999	0.692662568969323\\
37.00999	0.676901856455699\\
38.00999	0.695290171009879\\
39.00999	0.650509973174577\\
40.00999	0.691567613832743\\
41.00999	0.669632310621276\\
42.00999	0.662232045057157\\
43.00999	0.646393252353671\\
44.00999	0.61801528173026\\
45.00999	0.648148924703421\\
46.00999	0.634064370791283\\
47.00999	0.608626392240478\\
48.00999	0.638703154201612\\
49.00999	0.647077955527055\\
};
\addlegendentry{MH ($\sigma = 1.0\,m$)}

\addplot [color=mycolor3, line width=1.5pt]
  table[row sep=crcr]{%
0	10\\
1	8.45964379342644\\
2	6.48916806144404\\
3	4.83971090218149\\
4	4.01967189022007\\
5	3.58947245530542\\
6	3.16927698588751\\
7	2.83641803542556\\
8	2.73548549248234\\
9	2.68079917632438\\
10	2.33365401705688\\
11.00999	2.14340791608802\\
12.00999	2.07826404062335\\
13.00999	1.76079413739608\\
14.00999	1.73910403252804\\
15.00999	1.69358833250982\\
16.00999	1.39000633294247\\
17.00999	1.39005999446777\\
18.00999	1.29242636191549\\
19.00999	1.23597237044657\\
20.00999	1.17618255029825\\
21.00999	1.04567553579414\\
22.00999	1.03362901040116\\
23.00999	0.993496913044519\\
24.00999	0.983596193156641\\
25.00999	0.95404837118323\\
26.00999	0.908676901356342\\
27.00999	0.999846276763486\\
28.00999	0.954898693136438\\
29.00999	1.05313044724341\\
30.00999	0.9650171834698\\
31.00999	0.924483779964317\\
32.00999	0.894567166076881\\
33.00999	0.932045339234012\\
34.00999	0.948147933119518\\
35.00999	0.912476329683331\\
36.00999	0.979813067410036\\
37.00999	0.933023537897898\\
38.00999	0.959560069527999\\
39.00999	0.930622460147041\\
40.00999	1.01611113412896\\
41.00999	1.02978375897882\\
42.00999	0.954766586967735\\
43.00999	0.997816624514943\\
44.00999	0.927659446388562\\
45.00999	0.965300131505333\\
46.00999	0.980982768585791\\
47.00999	0.965939450973383\\
48.00999	0.946759877636097\\
49.00999	0.971364317573702\\
};
\addlegendentry{MH ($\sigma = 1.5\,m$)}

\addplot [color=mycolor4, line width=1.5pt]
  table[row sep=crcr]{%
0	10\\
1	8.66031437052369\\
2	6.86167007732462\\
3	5.29057213785325\\
4	4.69408889743197\\
5	4.28969169302897\\
6	3.79617929514222\\
7	3.41813862926426\\
8	3.12718502527695\\
9	2.8958073097385\\
10	2.78619976485208\\
11.00999	2.55420330684631\\
12.00999	2.35140755665467\\
13.00999	2.11314614788506\\
14.00999	1.97378048140227\\
15.00999	1.90523525239319\\
16.00999	1.80179740854344\\
17.00999	1.70525582600889\\
18.00999	1.61456243731158\\
19.00999	1.53075989240821\\
20.00999	1.42190195991826\\
21.00999	1.32710973563092\\
22.00999	1.31191111853325\\
23.00999	1.30364645083912\\
24.00999	1.25311584853998\\
25.00999	1.24275276403615\\
26.00999	1.15365161517747\\
27.00999	1.24094945962329\\
28.00999	1.18034024651067\\
29.00999	1.23309455556241\\
30.00999	1.16884381616764\\
31.00999	1.16031182063282\\
32.00999	1.14170291844947\\
33.00999	1.14680760411208\\
34.00999	1.17539428625077\\
35.00999	1.13779297050882\\
36.00999	1.17446996969669\\
37.00999	1.14768403898529\\
38.00999	1.15073885688751\\
39.00999	1.09946629167165\\
40.00999	1.14595237585168\\
41.00999	1.13634200793792\\
42.00999	1.12407401179165\\
43.00999	1.11692599861025\\
44.00999	1.1228878784843\\
45.00999	1.06910665593809\\
46.00999	1.06514114889028\\
47.00999	1.09815695190929\\
48.00999	1.07489857978567\\
49.00999	1.08377325289822\\
};
\addlegendentry{MH ($\sigma = 2.0\,m$)}

\addplot [color=mycolor1, dashed, line width=1.5pt]
  table[row sep=crcr]{%
0	10\\
1	8.12310118544798\\
2	5.62976897334151\\
3	3.97937274158116\\
4	3.14729693920519\\
5	2.78451753010945\\
6	2.48853425134191\\
7	2.26120490853862\\
8	2.00972351099719\\
9	1.73865736725091\\
10	1.62410991081765\\
11.00999	1.38547687785999\\
12.00999	1.22588021969496\\
13.00999	1.07870551594564\\
14.00999	1.00332608097574\\
15.00999	0.833215364043699\\
16.00999	0.80781265413529\\
17.00999	0.733259001387645\\
18.00999	0.694332849623581\\
19.00999	0.700875611571056\\
20.00999	0.612897524243052\\
21.00999	0.599138647020519\\
22.00999	0.63522362012115\\
23.00999	0.543769670239337\\
24.00999	0.56929072352576\\
25.00999	0.495232371890803\\
26.00999	0.520171774273679\\
27.00999	0.58774435716283\\
28.00999	0.506634748139636\\
29.00999	0.581724808368173\\
30.00999	0.49022446061074\\
31.00999	0.511470290759698\\
32.00999	0.505056938835447\\
33.00999	0.49397057391957\\
34.00999	0.510547230791844\\
35.00999	0.494860625338164\\
36.00999	0.471943101992993\\
37.00999	0.489322617408272\\
38.00999	0.477097752087556\\
39.00999	0.431401561095336\\
40.00999	0.430100868733537\\
41.00999	0.459633354374084\\
42.00999	0.440753852610173\\
43.00999	0.419989830771603\\
44.00999	0.425554137381047\\
45.00999	0.446709886197421\\
46.00999	0.421860898946388\\
47.00999	0.4319902120209\\
48.00999	0.439609359829132\\
49.00999	0.415966994880102\\
};
\addlegendentry{MHC ($\sigma = 0.5\,m$)}

\addplot [color=mycolor2, dashed, line width=1.5pt]
  table[row sep=crcr]{%
0	10\\
1	8.36507958633438\\
2	6.31311902639643\\
3	4.64889252609042\\
4	3.80691308059313\\
5	3.5380551038182\\
6	3.08052373827247\\
7	2.75775558587488\\
8	2.60414864165058\\
9	2.46625234029064\\
10	2.13893797826637\\
11.00999	1.98364918568689\\
12.00999	1.92674595194126\\
13.00999	1.53948126013787\\
14.00999	1.34176170557031\\
15.00999	1.16999980801336\\
16.00999	1.06925923950441\\
17.00999	1.08350768334293\\
18.00999	0.966368787322536\\
19.00999	0.927360721382871\\
20.00999	0.764624444041639\\
21.00999	0.820950175667008\\
22.00999	0.770784702220376\\
23.00999	0.751232002093374\\
24.00999	0.752670482711759\\
25.00999	0.71491670750557\\
26.00999	0.696372597333216\\
27.00999	0.775200945411803\\
28.00999	0.692973369095322\\
29.00999	0.765203618693328\\
30.00999	0.655332573828002\\
31.00999	0.663821368404469\\
32.00999	0.64477846323339\\
33.00999	0.639459717879491\\
34.00999	0.621967433704037\\
35.00999	0.646486017579449\\
36.00999	0.671873401037047\\
37.00999	0.639354268525536\\
38.00999	0.631641992991951\\
39.00999	0.627822766981423\\
40.00999	0.674408751536791\\
41.00999	0.640845680522155\\
42.00999	0.604961828142003\\
43.00999	0.615953861104885\\
44.00999	0.590233771336066\\
45.00999	0.61582109437156\\
46.00999	0.604392530582484\\
47.00999	0.585437395112894\\
48.00999	0.59745008821031\\
49.00999	0.569636356535555\\
};
\addlegendentry{MHC ($\sigma = 1.0\,m$}

\addplot [color=mycolor3, dashed, line width=1.5pt]
  table[row sep=crcr]{%
0	10\\
1	8.44430615944239\\
2	6.46618466988673\\
3	4.91888071813382\\
4	4.1338351046579\\
5	3.77965752607063\\
6	3.40294322736641\\
7	3.03111370840957\\
8	2.80429046286705\\
9	2.81777676884013\\
10	2.47212937839779\\
11.00999	2.33365543379814\\
12.00999	2.20384413898113\\
13.00999	2.02390874057737\\
14.00999	1.83542899118679\\
15.00999	1.72108662953859\\
16.00999	1.45402153156407\\
17.00999	1.40432211980113\\
18.00999	1.28014170458133\\
19.00999	1.14729753377975\\
20.00999	0.998586988057563\\
21.00999	0.963443962289036\\
22.00999	0.970960906506169\\
23.00999	0.942954484000027\\
24.00999	0.8895160742056\\
25.00999	0.857146024458145\\
26.00999	0.839094438998275\\
27.00999	0.933803780708737\\
28.00999	0.825381482887085\\
29.00999	0.934667965554437\\
30.00999	0.843042811962283\\
31.00999	0.81067218154948\\
32.00999	0.810000762291636\\
33.00999	0.845036163201678\\
34.00999	0.840623406384439\\
35.00999	0.804138708214545\\
36.00999	0.827992546050722\\
37.00999	0.819839066408749\\
38.00999	0.774926918063849\\
39.00999	0.753388667129504\\
40.00999	0.798145453295831\\
41.00999	0.742177029426385\\
42.00999	0.732583577465453\\
43.00999	0.725249563306302\\
44.00999	0.693119106631854\\
45.00999	0.732064646094724\\
46.00999	0.713180334967949\\
47.00999	0.688324405287318\\
48.00999	0.656463385549335\\
49.00999	0.670085354373725\\
};
\addlegendentry{MHC ($\sigma = 1.5\,m$)}

\addplot [color=mycolor4, dashed, line width=1.5pt]
  table[row sep=crcr]{%
0	10\\
1	8.66995244438989\\
2	6.91046770529451\\
3	5.44026161580221\\
4	4.79045095361375\\
5	4.38181251865178\\
6	3.86398610600605\\
7	3.51854216881031\\
8	3.19751393024064\\
9	3.15572322937598\\
10	2.9177800859349\\
11.00999	2.70814879749992\\
12.00999	2.4677675765897\\
13.00999	2.24226595601752\\
14.00999	2.04518183221997\\
15.00999	1.91275211994903\\
16.00999	1.76730294088324\\
17.00999	1.7341258580536\\
18.00999	1.63297959803493\\
19.00999	1.54069081692736\\
20.00999	1.39548176501882\\
21.00999	1.39759246369795\\
22.00999	1.31815814882388\\
23.00999	1.25613474450519\\
24.00999	1.22942116935667\\
25.00999	1.15083505183226\\
26.00999	1.16877121307988\\
27.00999	1.23556043783526\\
28.00999	1.12963572531215\\
29.00999	1.1858972459364\\
30.00999	1.10710442279775\\
31.00999	1.102841008601\\
32.00999	1.10443633344824\\
33.00999	1.067013696001\\
34.00999	1.05492339832304\\
35.00999	1.03636987981806\\
36.00999	1.07765096734117\\
37.00999	0.987743104602189\\
38.00999	1.03057639059834\\
39.00999	0.992190526019396\\
40.00999	1.04422138899561\\
41.00999	0.970120831139667\\
42.00999	0.955695704388345\\
43.00999	0.943658833306888\\
44.00999	0.938769460170456\\
45.00999	0.882455059416604\\
46.00999	0.902227944248594\\
47.00999	0.901635306028631\\
48.00999	0.849849786268688\\
49.00999	0.90954279932378\\
};
\addlegendentry{MHC ($\sigma = 2.0\,m$)}

\end{axis}
\end{tikzpicture}%

%% file: real_data_scenario.tex
\definecolor{mycolor1}{rgb}{0.46600,0.67400,0.18800}%
\definecolor{mycolor2}{rgb}{0.85000,0.32500,0.09800}%
\definecolor{mycolor3}{rgb}{0.92900,0.69400,0.12500}%
\definecolor{mycolor4}{rgb}{0.49400,0.18400,0.55600}%
\definecolor{mycolor5}{rgb}{0.00000,0.44700,0.74100}%
\begin{tikzpicture}

\begin{axis}[%
width=0.951\figW,
height=0.374\figW,
at={(0\figW,0\figW)},
scale only axis,
unbounded coords=jump,
xmin=-100,
xmax=40,
xlabel style={font=\color{white!15!black}},
xlabel={South-Direction / m},
ymin=-20,
ymax=35,
ylabel style={font=\color{white!15!black}},
ylabel={East-Direction / m},
axis background/.style={fill=white},
axis x line*=bottom,
axis y line*=left,
xmajorgrids,
ymajorgrids,
legend style={at={(0.03,0.97)}, anchor=north west, legend cell align=left, align=left, draw=white!15!black},
ticklabel style={font={\scriptsize}},legend style={font={\scriptsize}},legend columns={2},xlabel style={font={\scriptsize}},xlabel shift={-3pt},ylabel style={font={\scriptsize}},ylabel shift={-9pt}
]
\addplot [color=mycolor1, line width=1.5pt]
  table[row sep=crcr]{%
-146.814325442538	-10.4237954017008\\
-146.369665089063	-10.4167455764255\\
-144.813350907527	-10.3916001081234\\
-143.144957793877	-10.4385398118757\\
-142.033293560147	-10.4187793299789\\
-141.588632587343	-10.4107805490494\\
-140.03134950716	-10.4515231400728\\
-138.808606455103	-10.4272683724994\\
-138.363986906596	-10.4185558961472\\
-136.807880341075	-10.3869931754889\\
-135.696414049715	-10.3665194519563\\
-135.251821403392	-10.3594688426238\\
-133.696683278307	-10.2615421379451\\
-132.029408650473	-10.2397376167355\\
-130.363062275574	-10.1444424609654\\
-129.029245328158	-10.126856226474\\
-128.584639748558	-10.1209950584453\\
-126.695045094937	-10.0973281932529\\
-125.250987935811	-10.0057950065238\\
-124.806374545209	-10.0004073986784\\
-123.250229801051	-9.98143510823138\\
-122.805612943135	-9.97628565179184\\
-120.694646080025	-9.87607768340968\\
-118.582723797299	-9.85126158583444\\
-117.026571927592	-9.83300106483512\\
-116.470807045698	-9.82620873383712\\
-114.35981791839	-9.7281386986142\\
-112.803670535795	-9.70987821964081\\
-112.247902179137	-9.70356086373795\\
-110.024834100157	-9.67805194715038\\
-108.469622260891	-9.58677426585928\\
-107.913855842315	-9.58069373981562\\
-105.690798099153	-9.55565986293368\\
-104.134661701508	-9.53858760639559\\
-103.578899926506	-9.5325070945546\\
-101.35586027801	-9.50842321990058\\
-99.6885930430144	-9.49113298777957\\
-99.1328357663006	-9.48481429065578\\
-97.0200292356312	-9.53515383112244\\
-95.3527670614421	-9.51643723365851\\
-94.7970213107765	-9.51083299610764\\
-92.5740038435906	-9.485085860244\\
-90.4620928559452	-9.45718202367425\\
-88.79387750756	-9.51053304411471\\
-88.2381187183782	-9.50374073872808\\
-86.5708525879309	-9.48407419386785\\
-86.0151045899838	-9.47799361206125\\
-83.7911669788882	-9.52669025515206\\
-82.1239106487483	-9.50726043409668\\
-81.5681619821116	-9.50094302836806\\
-79.4553539026529	-9.54938286880497\\
-77.2323156790808	-9.5217360224342\\
-75.00828669779	-9.56520607776474\\
-72.7851784098893	-9.53352180239744\\
-70.5611500609666	-9.57628005067818\\
-68.4491837266833	-9.54315185290761\\
-66.2260982813314	-9.5093291135272\\
-64.0030368259177	-9.47574445069768\\
-61.8911424139515	-9.44665396516211\\
-59.6680954629555	-9.4180573632475\\
-57.4450486879796	-9.38945940916892\\
-55.8889153683558	-9.36977429175749\\
-55.3331534685567	-9.36274389340542\\
-53.1101100826636	-9.33296053041704\\
-51.4428229834884	-9.31234433234204\\
-50.9982128264382	-9.30695683055092\\
-48.7751665748656	-9.27836027310695\\
-46.663275424391	-9.25045800465159\\
-44.4402307868004	-9.22162329370622\\
-42.3283389080316	-9.19419466983527\\
-40.6619952600449	-9.09794863837305\\
-40.1062347665429	-9.09068007674068\\
-38.5501074567437	-9.0700450274162\\
-37.9943488053977	-9.0625396608375\\
-35.7713119667023	-9.03299318579957\\
-34.2151846094057	-9.01259496831335\\
-33.6594247743487	-9.00532777898479\\
-31.4373109089211	-8.90347482869402\\
-29.3254117937759	-8.87723432760686\\
-27.2135074371472	-8.85170697933063\\
-24.9904502397403	-8.82501037919428\\
-22.8785433713347	-8.80043302872218\\
-21.3224041648209	-8.78217292018235\\
-20.7666414417326	-8.77538071572781\\
-19.0993535378948	-8.75500138255302\\
-18.6547421142459	-8.75009024201427\\
-16.9883911339566	-8.65550569165498\\
-16.54378464818	-8.64869323640596\\
-14.876540816389	-8.62023921287619\\
-14.4319533063099	-8.6122387049254\\
-12.2090646587312	-8.57295406993944\\
-10.0984899373725	-8.45636070519686\\
-7.98815134260803	-8.33858441980556\\
-5.87810672167689	-8.21939583914354\\
-4.32384789455682	-8.11056887882296\\
-3.8805529139936	-8.02552470995579\\
-2.3268835823983	-7.91245390940458\\
-1.77277255058289	-7.82341044710483\\
0.220958091318607	-7.53675086633302\\
1.76986721996218	-7.32982345693745\\
2.322300568223	-7.2335704092402\\
3.75430030561984	-6.86870769772213\\
4.30523733142763	-6.76910863642115\\
5.73186576925218	-6.39421839814167\\
6.16851444169879	-6.2184357115766\\
7.58683299738914	-5.68876690464094\\
8.13279585260898	-5.51115504908375\\
9.86476978100836	-4.65664062858559\\
11.1594760576263	-3.98179427848663\\
11.5901687229052	-3.7329991691513\\
13.3091931799427	-2.66731615341268\\
14.4846360972151	-1.78231794969179\\
14.913001188077	-1.46318456588779\\
16.4017044883221	-0.191879139165394\\
17.4583966974169	0.833233548211865\\
17.8843026291579	1.15098893654067\\
19.2496041934937	2.63671817397699\\
20.1870992230251	3.80263635562733\\
20.4991364562884	4.190746816108\\
21.6324561070651	5.88670110271778\\
22.451482038945	7.1193630786147\\
22.7605111654848	7.57942233688664\\
23.6641979012638	9.41072421602439\\
24.5751860952005	11.2326172266621\\
25.3827963592485	13.192586061894\\
25.883294862695	14.6200579573633\\
25.9765665261075	15.1435383385979\\
26.692265542224	17.0901501397602\\
27.3045626832172	19.1053664461942\\
27.8093724679202	21.1173767283326\\
28.321361053735	23.1268622180214\\
28.8341946462169	25.1358440284384\\
29.2499647829682	26.623479181435\\
29.3518022885546	27.1433721251087\\
29.7612248538062	29.1486351160565\\
30.1797236092389	30.6354343845742\\
30.2820125296712	31.1551757248817\\
30.6925546424463	33.1600258226972\\
31.0012658052146	34.6450021938654\\
31.1037798393518	35.1646617117804\\
31.5152313029394	37.0954312126851\\
31.9241067022085	39.1004526751349\\
32.2215845333412	41.030249450705\\
32.5276943473145	42.441758860019\\
32.6285585919395	42.9616630568635\\
32.9238895364106	44.8178190751933\\
33.220360968262	46.5998052188661\\
33.4077621446922	48.3060746401316\\
33.6056986581534	49.6414902947145\\
33.709119386971	50.0870299169328\\
33.7987154722214	51.4205757862655\\
33.9030733052641	51.8659688069019\\
34.1064053000882	53.2005326156504\\
34.1003162059933	53.6444146574941\\
34.299440279603	55.4225884382613\\
34.5034808628261	56.8308830941096\\
34.6090140230954	57.2761050707195\\
34.8137588333338	58.6842994840117\\
34.8083752794191	59.1280805456918\\
35.0145301586017	60.5360727481311\\
35.1214747047052	60.9810909113148\\
35.2185916081071	62.3873610631563\\
35.3252996364608	62.8324089219095\\
35.4273272417486	63.7211295970483\\
35.6402710946277	64.611290208064\\
35.8429158050567	66.4627579030348\\
36.0462640887126	68.3141039112816\\
36.3609930695966	70.0929873951245\\
36.5701891351491	71.426637958386\\
36.6771272690967	71.8716293751495\\
36.8818704299629	73.6488837382058\\
37.2036176836118	74.9836725300411\\
37.2003370691091	75.4270815370837\\
37.5225359946489	77.2045586982276\\
37.7398885982111	78.4627950905124\\
37.8496174141765	78.9072071947157\\
38.1810844419524	80.6087634960422\\
38.4055778682232	81.7913621587213\\
38.517133471556	82.2352239727043\\
38.9647045824677	83.8625741265714\\
39.3028918839991	85.3401599086355\\
39.7539244433865	86.8919211716857\\
};
\addlegendentry{Ground truth}

\addplot [color=black, line width=0.8pt, draw=none, mark=o, mark options={solid, black}, forget plot]
  table[row sep=crcr]{%
-146.814325442538	-10.4237954017008\\
};
\addplot [color=black, line width=0.8pt, draw=none, mark=triangle, mark options={solid, black}, forget plot]
  table[row sep=crcr]{%
39.7539244433865	86.8919211716857\\
};
\addplot [color=mycolor2, dashdotted, line width=1.5pt]
  table[row sep=crcr]{%
nan	nan\\
nan	nan\\
nan	nan\\
nan	nan\\
nan	nan\\
nan	nan\\
nan	nan\\
nan	nan\\
nan	nan\\
nan	nan\\
nan	nan\\
nan	nan\\
nan	nan\\
nan	nan\\
nan	nan\\
nan	nan\\
nan	nan\\
nan	nan\\
nan	nan\\
nan	nan\\
nan	nan\\
nan	nan\\
nan	nan\\
nan	nan\\
nan	nan\\
nan	nan\\
nan	nan\\
nan	nan\\
nan	nan\\
nan	nan\\
nan	nan\\
nan	nan\\
nan	nan\\
nan	nan\\
nan	nan\\
nan	nan\\
nan	nan\\
nan	nan\\
nan	nan\\
nan	nan\\
nan	nan\\
nan	nan\\
nan	nan\\
nan	nan\\
nan	nan\\
nan	nan\\
nan	nan\\
nan	nan\\
-81.21	-10.43\\
-78.05	-7.57\\
-79.64	-7.6\\
-77.37	-7.6\\
-75.12	-7.6\\
-72.76	-7.77\\
-70.56	-7.83\\
-68.44	-7.87\\
-66.25	-7.81\\
-63.98	-7.79\\
-61.71	-7.82\\
-59.44	-7.85\\
-57.11	-7.93\\
-55.09	-7.94\\
-54.64	-7.89\\
-52.66	-7.97\\
-51.19	-8.1\\
-50.71	-8.1\\
-49.04	-8.16\\
-46.89	-8.2\\
-44.68	-8.27\\
-42.46	-8.24\\
-40.83	-8.21\\
-40.33	-8.16\\
-38.63	-8.18\\
-38.08	-8.14\\
-36.1	-8.22\\
-34.41	-8.25\\
-33.97	-8.2\\
-31.91	-8.27\\
-30.02	-8.24\\
-28.45	-8.14\\
-26.02	-8.18\\
-24.05	-8.21\\
-22.41	-8.2\\
-21.85	-8.2\\
-20.16	-8.15\\
-19.58	-8.16\\
-18.06	-8.12\\
-17.6	-8.12\\
-15.91	-8.12\\
-15.34	-8.15\\
-13.1	-8.21\\
-10.87	-8.13\\
-8.68	-7.91\\
-6.49	-7.59\\
-4.87	-7.31\\
-4.33	-7.54\\
-2.68	-7.69\\
-2.16	-7.79\\
-0.01	-7.41\\
1.6	-7.16\\
2.2	-7.03\\
3.83	-7.11\\
4.35	-7.11\\
5.87	-6.88\\
6.39	-6.65\\
7.9	-6.44\\
8.41	-6.19\\
10.45	-5.32\\
11.92	-4.69\\
12.38	-4.34\\
13.93	-3.16\\
15.13	-2.28\\
15.42	-1.91\\
16.86	-0.59\\
17.92	0.42\\
18.33	0.7\\
19.79	1.97\\
20.77	3.06\\
21.07	3.4\\
22.19	4.97\\
22.96	6.17\\
23.17	6.59\\
24.06	8.25\\
24.85	9.93\\
25.49	11.69\\
25.9	13.03\\
26.03	13.5\\
26.45	15.37\\
26.79	17.21\\
27.19	19.13\\
};
\addlegendentry{Vehicle 1 (UT)}

\addplot [color=mycolor3, dotted, line width=1.5pt]
  table[row sep=crcr]{%
nan	nan\\
nan	nan\\
nan	nan\\
nan	nan\\
nan	nan\\
nan	nan\\
nan	nan\\
nan	nan\\
nan	nan\\
nan	nan\\
nan	nan\\
nan	nan\\
nan	nan\\
nan	nan\\
nan	nan\\
nan	nan\\
nan	nan\\
nan	nan\\
nan	nan\\
nan	nan\\
nan	nan\\
nan	nan\\
nan	nan\\
nan	nan\\
nan	nan\\
nan	nan\\
nan	nan\\
nan	nan\\
nan	nan\\
nan	nan\\
nan	nan\\
nan	nan\\
nan	nan\\
nan	nan\\
nan	nan\\
nan	nan\\
nan	nan\\
nan	nan\\
nan	nan\\
nan	nan\\
nan	nan\\
nan	nan\\
nan	nan\\
nan	nan\\
nan	nan\\
nan	nan\\
nan	nan\\
nan	nan\\
nan	nan\\
nan	nan\\
nan	nan\\
nan	nan\\
nan	nan\\
nan	nan\\
nan	nan\\
nan	nan\\
nan	nan\\
nan	nan\\
nan	nan\\
nan	nan\\
nan	nan\\
nan	nan\\
nan	nan\\
nan	nan\\
-80.49	-6.4\\
-81.23	-7.71\\
-79.23	-7.68\\
-76.99	-7.53\\
-74.55	-7.5\\
-72.03	-7.64\\
-70.24	-7.75\\
-69.63	-7.73\\
-67.84	-7.81\\
-67.23	-7.74\\
-64.8	-7.66\\
-62.92	-7.73\\
-62.29	-7.7\\
-59.86	-7.74\\
-57.45	-7.78\\
-55.04	-7.76\\
-52.71	-7.88\\
-50.42	-7.96\\
-48.66	-8.04\\
-48.07	-8.05\\
-46.3	-8.19\\
-45.72	-8.21\\
-43.96	-8.23\\
-43.34	-8.19\\
-41.48	-8.21\\
-40.87	-8.18\\
-38.48	-8.27\\
-36.19	-8.18\\
-33.79	-8.04\\
-31.35	-7.89\\
-29.55	-7.9\\
-29.02	-7.87\\
-27.23	-7.9\\
-26.65	-7.82\\
-24.42	-7.59\\
-22.79	-7.51\\
-22.28	-7.49\\
-20.7	-7.42\\
-20.15	-7.4\\
-18.57	-7.33\\
-18.03	-7.35\\
-16.34	-7.35\\
-15.82	-7.38\\
-13.9	-7.43\\
-12.36	-7.41\\
-11.83	-7.53\\
-9.85	-7.73\\
-8.34	-7.69\\
-7.8	-7.76\\
-5.8	-7.7\\
-4.34	-7.43\\
-3.82	-7.59\\
-2	-7.8\\
-0.67	-7.76\\
-0.22	-7.86\\
1.45	-7.86\\
2.63	-7.76\\
3.06	-7.82\\
4.53	-7.79\\
5.91	-7.58\\
7.15	-7.77\\
8.12	-7.98\\
8.44	-8.13\\
9.69	-8.53\\
10.83	-9.07\\
11.84	-9.67\\
12.76	-10.33\\
13.72	-10.92\\
14.4	-11.37\\
14.62	-11.6\\
15.42	-12.44\\
15.92	-13.03\\
16.04	-13.42\\
16.5	-14.63\\
16.7	-15.39\\
16.75	-15.99\\
};
\addlegendentry{Vehicle 2}

\addplot [color=mycolor4, dashed, line width=1.5pt]
  table[row sep=crcr]{%
nan	nan\\
nan	nan\\
nan	nan\\
nan	nan\\
nan	nan\\
nan	nan\\
nan	nan\\
nan	nan\\
nan	nan\\
nan	nan\\
nan	nan\\
nan	nan\\
nan	nan\\
nan	nan\\
nan	nan\\
nan	nan\\
nan	nan\\
nan	nan\\
nan	nan\\
nan	nan\\
nan	nan\\
nan	nan\\
nan	nan\\
nan	nan\\
nan	nan\\
nan	nan\\
nan	nan\\
nan	nan\\
nan	nan\\
nan	nan\\
nan	nan\\
nan	nan\\
nan	nan\\
nan	nan\\
nan	nan\\
nan	nan\\
nan	nan\\
nan	nan\\
nan	nan\\
nan	nan\\
nan	nan\\
nan	nan\\
nan	nan\\
nan	nan\\
nan	nan\\
nan	nan\\
nan	nan\\
nan	nan\\
nan	nan\\
nan	nan\\
nan	nan\\
nan	nan\\
nan	nan\\
nan	nan\\
nan	nan\\
nan	nan\\
nan	nan\\
nan	nan\\
nan	nan\\
nan	nan\\
nan	nan\\
nan	nan\\
nan	nan\\
nan	nan\\
nan	nan\\
nan	nan\\
nan	nan\\
nan	nan\\
nan	nan\\
nan	nan\\
nan	nan\\
nan	nan\\
nan	nan\\
nan	nan\\
nan	nan\\
nan	nan\\
nan	nan\\
-82.12	-7.86\\
-80.31	-7.8\\
-78.27	-7.75\\
-76.02	-7.76\\
-73.44	-7.84\\
-71.76	-7.88\\
-71.19	-7.89\\
-69.54	-7.92\\
-68.94	-7.92\\
-67.35	-7.96\\
-66.82	-7.97\\
-65.19	-7.98\\
-64.6	-8\\
-62.39	-8.03\\
-60.23	-8.04\\
-58.16	-8.08\\
-56.12	-8.09\\
-54.53	-8.19\\
-54.03	-8.24\\
-52.55	-8.44\\
-52.08	-8.36\\
-50.17	-8.32\\
-48.7	-8.3\\
-48.21	-8.32\\
-46.7	-8.39\\
-46.19	-8.4\\
-44.66	-8.51\\
-44.15	-8.49\\
-42.55	-8.65\\
-42.01	-8.59\\
-39.97	-8.36\\
-38.46	-8.3\\
-37.96	-8.26\\
-35.84	-8.19\\
-34.31	-8.17\\
-33.86	-8.14\\
-31.88	-8.11\\
-30.31	-8.1\\
-29.84	-8.07\\
-27.87	-8.1\\
-26.32	-8.09\\
-25.81	-8.08\\
-23.82	-8.02\\
-22.4	-7.94\\
-21.93	-7.96\\
-19.97	-7.93\\
-18.05	-7.91\\
-16.19	-7.89\\
-14.64	-7.85\\
-14.13	-7.9\\
-12.06	-7.82\\
-9.89	-7.92\\
-7.69	-7.95\\
-5.48	-7.85\\
-3.41	-7.57\\
-1.76	-7.42\\
-1.24	-7.39\\
1.02	-7.03\\
2.73	-7.14\\
3.29	-6.97\\
5.46	-6.7\\
7.06	-6.49\\
7.56	-6.26\\
9.54	-5.7\\
11.46	-4.95\\
13.27	-3.6\\
14.54	-2.59\\
14.88	-2.16\\
16.41	-0.72\\
17.86	0.73\\
19.38	2.1\\
20.51	3.16\\
20.86	3.49\\
21.86	4.65\\
22.14	5.06\\
22.99	6.33\\
23.22	6.76\\
24.18	8.47\\
24.87	9.79\\
25.09	10.23\\
25.68	11.59\\
25.86	12.03\\
26.34	13.43\\
26.48	13.9\\
26.85	15.33\\
26.95	15.8\\
27.11	16.73\\
27.25	17.68\\
27.58	19.61\\
};
\addlegendentry{Vehicle 3}

\addplot [color=black, line width=0.8pt, draw=none, mark=o, mark options={solid, black}, forget plot]
  table[row sep=crcr]{%
-81.21	-10.43\\
};
\addplot [color=black, line width=0.8pt, draw=none, mark=triangle, mark options={solid, black}, forget plot]
  table[row sep=crcr]{%
27.19	19.13\\
};
\addplot [color=black, line width=0.8pt, draw=none, mark=o, mark options={solid, black}, forget plot]
  table[row sep=crcr]{%
-80.49	-6.4\\
};
\addplot [color=black, line width=0.8pt, draw=none, mark=triangle, mark options={solid, black}, forget plot]
  table[row sep=crcr]{%
16.75	-15.99\\
};
\addplot [color=black, line width=0.8pt, draw=none, mark=o, mark options={solid, black}, forget plot]
  table[row sep=crcr]{%
-82.12	-7.86\\
};
\addplot [color=black, line width=0.8pt, draw=none, mark=triangle, mark options={solid, black}, forget plot]
  table[row sep=crcr]{%
27.58	19.61\\
};
\addplot [color=mycolor5, line width=1.5pt, draw=none, mark size=6.0pt, mark=x, mark options={solid, mycolor5}, forget plot]
  table[row sep=crcr]{%
-37.633929499425	-14.2066714136163\\
};
\node[right, align=left]
at (axis cs:-34.634,-16.207) {(C,D)};
\addplot [color=mycolor5, line width=1.5pt, draw=none, mark size=6.0pt, mark=x, mark options={solid, mycolor5}, forget plot]
  table[row sep=crcr]{%
4.81229790206999	-10.8067196659977\\
};
\node[right, align=left]
at (axis cs:-16.188,-13.807) {(A,E)};
\addplot [color=mycolor5, line width=1.5pt, draw=none, mark size=6.0pt, mark=x, mark options={solid, mycolor5}, forget plot]
  table[row sep=crcr]{%
21.0814605774358	23.0677262848476\\
};
\node[right, align=left]
at (axis cs:5.081,20.068) {(F)};
\addplot [color=mycolor5, line width=1.5pt, draw=none, mark size=6.0pt, mark=x, mark options={solid, mycolor5}]
  table[row sep=crcr]{%
-81.6543142646551	-13.5650241791736\\
};
\addlegendentry{sensor positions}

\node[right, align=left]
at (axis cs:-78.654,-15.565) {(B,G)};
\end{axis}
\end{tikzpicture}%